\documentclass{article}

\usepackage[preprint]{neurips_2026}

\usepackage[utf8]{inputenc} 
\usepackage[T1]{fontenc}    
\usepackage{microtype}
\usepackage{graphicx}
\usepackage{subcaption}
\usepackage{booktabs} 
\usepackage{hyperref}
\usepackage{url}
\usepackage{amsfonts}
\usepackage{nicefrac}
\usepackage{xcolor,colortbl}

\usepackage{amsmath}
\usepackage{amssymb}
\usepackage{mathtools}
\usepackage{amsthm}
\usepackage[most]{tcolorbox}
\usepackage{listings}
\usepackage{multirow}
\usepackage{tabularx}
\usepackage{makecell}
\usepackage{CJKutf8}
\usepackage{wrapfig}
\usepackage{enumitem}
\usepackage{caption}
\usepackage{marvosym}

\definecolor{codegreen}{rgb}{0,0.6,0}
\definecolor{codegray}{rgb}{0.5,0.5,0.5}
\definecolor{codepurple}{rgb}{0.58,0,0.82}
\definecolor{backcolour}{rgb}{0.95,0.95,0.92}

\lstdefinestyle{mystyle}{
    commentstyle= \color{red!50!green!50!blue!50},
    keywordstyle= \color{blue!70},
    numberstyle=\tiny\color{codegray},
    stringstyle=\color{codepurple},
    basicstyle=\ttfamily\footnotesize,
    breakatwhitespace=false,
    breaklines=true,
    captionpos=b,
    keepspaces=true,
    numbers=left,
    numbersep=-5pt,
    showspaces=false,
    showstringspaces=false,
    showtabs=false,
    tabsize=2,
    frame=single
}

\usepackage[capitalize,noabbrev]{cleveref}

\theoremstyle{plain}
\newtheorem{theorem}{Theorem}[section]

\theoremstyle{definition}
\newtheorem{definition}[theorem]{Definition}

\theoremstyle{remark}

\definecolor{darkgreen}{RGB}{0,128,0}

\usepackage[textsize=tiny]{todonotes}

\newcommand{\rebuttal}[1]{#1}

\title{BEAM: Binary Expert Activation Masking for Dynamic Routing in MoE}

\author{
    Juntong Wu\textsuperscript{1,2,*},
    Jialiang Cheng\textsuperscript{1,*, \Letter},
    Qishen Yin\textsuperscript{2},
    Yue Dai \textsuperscript{1}, \And \vspace{1mm}
    Yuliang Yan\textsuperscript{1},
    Fuyu Lv\textsuperscript{1},
    Ou Dan\textsuperscript{1},
    Li Yuan\textsuperscript{2, \Letter} \\ 
    \textsuperscript{1} Taobao \& Tmall Group of Alibaba \\
    \textsuperscript{2} Shenzhen Graduate School, Peking University \\
    \small{
   \textbf{Correspondence:} \href{mailto:jichen.cjl@alibaba-inc.com}{jichen.cjl@alibaba-inc.com},
   \href{mailto:yuanli-ece@pku.edu.cn}{yuanli-ece@pku.edu.cn} }
    \thanks{\textsuperscript{*} Equal contribution  \quad \textsuperscript{\Letter} \ Corresponding author}
}

\makeatletter
\def\thanks#1{\protected@xdef\@thanks{\@thanks
        \protect\footnotetext{#1}}}
\makeatother

\begin{document}

\maketitle

\vspace{-5mm}
\begin{abstract}

Mixture-of-Experts (MoE) architectures enhance the efficiency of large language models by activating only a subset of experts per token. However, standard MoE employs a fixed Top-K routing strategy, leading to redundant computation and suboptimal inference latency.
Existing acceleration methods either require costly retraining with architectural changes or suffer from severe performance drop at high sparsity due to train-inference mismatch.
To address these limitations, we propose \textbf{BEAM} (Binary Expert Activation Masking), a novel method that learns token-adaptive expert selection via trainable binary masks. With a straight-through estimator and an auxiliary regularization loss, BEAM induces dynamic expert sparsity through end-to-end training while maintaining model capability. We further implement an efficient custom CUDA kernel for BEAM, ensuring seamless integration with the vLLM inference framework.
Experiments show that BEAM retains over 98\% of the original model's performance while reducing MoE layer FLOPs by up to 85\%, achieving up to 2.5$\times$ faster decoding and 1.4$\times$ higher throughput, demonstrating its effectiveness as a practical, plug-and-play solution for efficient MoE inference. Code implementation of BEAM can be found in~\url{https://github.com/Time-Rune/BEAM}.

\end{abstract}

\section{Introduction}


Mixture-of-Experts (MoE) enables efficient scaling through sparse activation, where each token is processed by only a small subset of specialized feed-forward network (FFN) experts~\citep{yang2025qwen3,liu2024deepseek,jiang2024mixtral}.
\begin{wrapfigure}{r}{0.52\textwidth}
    \vspace{-1mm}
    \centering
    \captionsetup{font=small}
    \includegraphics[width=\linewidth]{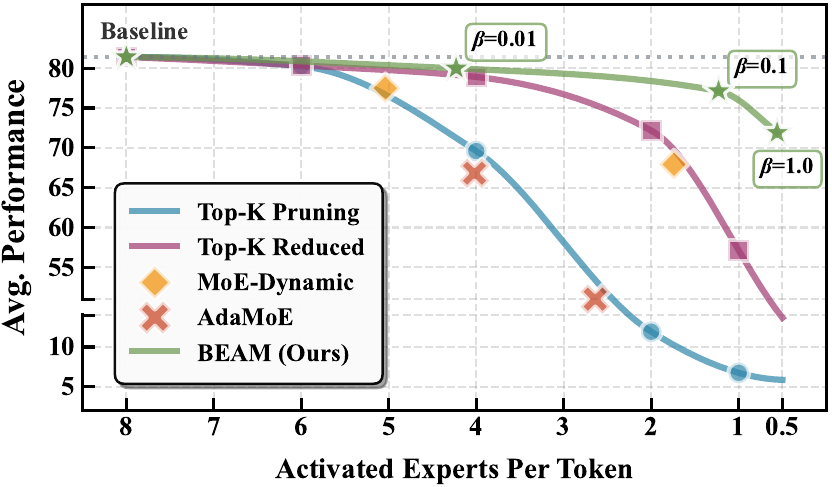}
    \vspace{-6mm}
    \caption{\rebuttal{Performance--sparsity trade-off of BEAM and baselines on Qwen3-30B-A3B.}}
    \label{fig:performance_comparison}
    \vspace{-6mm}
\end{wrapfigure}
The dominant paradigm for expert selection is the fixed Top-K routing mechanism, which selects the K experts with the highest router logits for each token~\citep{shazeer2017outrageously,lepikhin2020gshard}.
While simple and widely adopted, it ignores token-level complexity, leading to redundant computation for simple tokens~\citep{huang2024harder,zeng2024adamoe}.
This inefficiency ultimately limits the potential for faster MoE model inference.


To address the inefficiency of fixed Top-K routing, recent work has explored dynamic expert activation, falling into three categories.
The first modifies the routing logits to enable token-adaptive expert counts~\citep{huang2024harder,lu2024not,yang2024xmoe,aghdam2024moe,guo2024dynamic}, but fails to skip redundant high-weight experts and enforces a minimum activation floor, limiting achievable sparsity.
The second introduces special experts such as zero-computation null experts to control sparsity~\citep{zeng2024adamoe,jin2024moe++,team2025introducing}, yet requires additional hyperparameters and complicated fine-tuning process, and only enables passive, indirect sparsity control.
The third merges or prunes experts statically~\citep{chen2025retrainingfree,liu2024efficient,yang2024moe}, but cannot adapt to input complexity at inference time and often suffers from severe performance degradation at high sparsity levels.

\begin{wrapfigure}{r}{0.54\textwidth}
    \vspace{-4mm}
    \centering
    \captionsetup{font=small}
    \includegraphics[width=\linewidth]{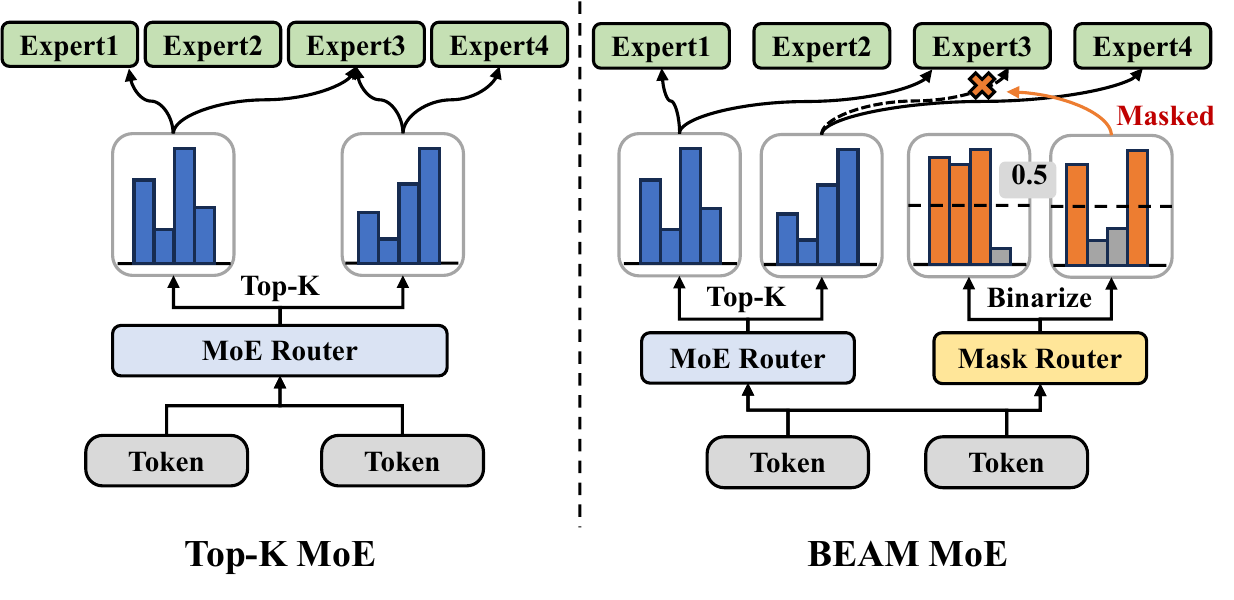}
    \vspace{-6mm}
    \caption{\rebuttal{Vanilla Top-$K$ vs.\ BEAM: BEAM learns a binary mask over Top-$K$ candidates for token-adaptive activation.}}
    \label{fig:topk_compare}
    \vspace{-1mm}
\end{wrapfigure}

In this work, we propose \textbf{BEAM} (\textbf{B}inary \textbf{E}xpert \textbf{A}ctivation \textbf{M}asking), a novel dynamic routing framework designed to achieve extreme expert sparsity and inference speedups in MoE models. As shown in Figure~\ref{fig:topk_compare}, BEAM introduces a lightweight, learnable mask router that generates a binary mask applied to the top-K candidate experts from the primary router, selectively deactivating redundant ones. 
Sparsity is encouraged via an auxiliary regularization loss, and gradients are propagated through the binary mask using the straight-through estimator (STE)~\citep{bengio2013estimating}.
Crucially, BEAM decouples sparsity control from expert selection. The primary router still handles load balancing and expert choice, while the mask router solely determines activation count. This separation avoids conflicts and enables more activation patterns within the Top-K candidate set, providing fine-grained, token-adaptive sparsity control that fixed Top-K or logits-based methods cannot express. To demonstrate the practical impact, we integrate BEAM into vLLM~\citep{kwon2023efficient}
through a custom CUDA kernel, requiring only a single-line change and delivering significant real-world speedups, which makes BEAM a practical, plug-and-play solution for efficient MoE deployment.





Our contributions are summarized as follows:
\vspace{-2mm}
\begin{itemize}[leftmargin=6mm]
\item We propose BEAM, a novel dynamic routing framework that achieves extreme expert sparsity via a learnable mask router. It directly prunes redundant experts from the Top-K set for token-adaptive computation, in contrast to existing indirect or post-hoc approaches.
\item We provide a practical, plug-and-play  deployment solution by integrating BEAM into vLLM through a custom CUDA kernel, requiring minimal code changes. 
\item Extensive experiments show BEAM preserves over 98\% of performance while reducing MoE layer FLOPs by up to 85\% (Figure~\ref{fig:performance_comparison}), yielding 1.4$\times$ higher throughput and 2.5$\times$ faster decoding.
\end{itemize}

\section{Related Work}


\noindent\textbf{Routing Logits Modification~} 
These methods modify routing logits to enable token-adaptive expert counts. MoE-Dynamic~\citep{huang2024harder} and XMoE~\citep{yang2024xmoe} activate experts until the cumulative probability exceeds a threshold. 
DTop-p~\citep{jin2025sparsity} improves MoE-Dynamic by replacing the fixed threshold with a learnable sparsity controller. 
Adaptive Gating~\citep{li2023adaptive} and NAEE~\citep{lu2024not} dynamically switches between Top-1 and Top-2 based on the gap between the top two logits.
DA-MoE~\citep{aghdam2024moe} computes token importance from attention scores to allocate a dynamic Top-K. 
DynMoE~\citep{guo2024dynamic} replaces the softmax router with per-expert sigmoid gates.
\rebuttal{MaskMoE~\citep{su2024maskmoe} employs static vocabulary-based masks derived from pretraining data distributions to improve rare-token expert assignment.}
However, most of them rely on the unverified heuristic that low entropy of routing logits implies fewer needed experts, fail to skip redundant high-weight experts, and require at least one active expert, preventing acceleration.

\noindent\textbf{Special Experts~} 
These methods reduce FLOPs by routing tokens to experts that incur no computation. AdaMoE~\citep{zeng2024adamoe} introduces null experts that outputs zero. 
LongCat~\citep{team2025introducing} uses zero-computation experts that return the input as their output.
MoE++~\citep{jin2024moe++} extended this idea with three types of zero-computation experts.
However, these methods introduce extra hyperparameters and achieve sparsity indirectly via passive placeholder routing rather than explicit expert minimization, undermining plug-and-play usability.

\noindent\textbf{Static Expert Merging and Pruning~} 
These training-free methods reduce redundancy by merging or pruning experts.
DEK~\citep{zhang2025diversifying} groups similar experts in feature space and merges experts in each group. 
EEP~\citep{liu2024efficient} utilizes a gradient-free evolutionary search to determine pruning and merging patterns. 
MC-SMoE~\citep{li2023merge} leverages routing statistics to guide expert merging and decomposes the merged experts into low-rank and structural sparse alternatives.
HC-SMoE~\citep{chen2025retrainingfree} applies hierarchical clustering on expert outputs to merge experts.
However, these methods cannot adapt to the varying complexity of input tokens at inference time and often suffer performance degradation under high compression.

\section{Method}

\subsection{Preliminaries and Motivation}
MoE replaces dense FFN layers with $N$ expert networks $\{\mathcal{E}_1, \dots, \mathcal{E}_N\}$ and a router $\mathcal{R}$. Given an input token $\mathbf{x} \in \mathbb{R}^{d_h}$, the router computes logits $\mathbf{r} = \mathcal{R}(\mathbf{x}) \in \mathbb{R}^N$, which are converted into routing weights via softmax. Under standard Top-K routing, only the K experts with the largest routing logits are activated. Specifically, the $\mathrm{Top}\text{-}K(\cdot)$ operator retains the K largest values in $\mathbf{r}$ and sets the remaining  entries to $-\infty$, yielding routing weights:
\begin{equation}
    \mathbf{g}_i = \mathrm{Softmax}(\mathrm{Top}\text{-}K(\mathbf{r}))_i.
\end{equation}
The MoE output is a weighted sum of expert outputs:
\begin{equation}
    \mathbf{y} = \sum_{i=1}^N \mathbf{g}_i \cdot \mathcal{E}_i(\mathbf{x}),
\end{equation}
where each expert $\mathcal{E}_i$ typically follows a Gated Linear Unit (GLU) structure:
\begin{equation}
    \mathcal{E}_i(\mathbf{x}) = \left( \delta(\mathbf{x} \mathbf{W}_{\mathrm{gate}}^{(i)}) \odot (\mathbf{x} \mathbf{W}_{\mathrm{up}}^{(i)}) \right) \mathbf{W}_{\mathrm{down}}^{(i)}.
\end{equation}

Although Top-K routing enables scalable training, it assigns a uniform computational budget to all tokens, causing redundancy for simple ones.

\rebuttal{Existing dynamic routing methods attempt to address this problem but remain limited in practice.
First, these approaches implicitly treat routing rank as a proxy for expert importance. However, a lower-ranked expert can still be critical for a given token while a high-weight one may be redundant, which is empirically validated in Section~\ref{sec:position_masking} and Appendix~\ref{sec:layerwise_rank}.
Second, cumulative probability thresholds and null experts cannot actively prune redundant experts, limiting compression ratios (Section~\ref{sec:perf_comparison}).
Third, these methods entangle expert selection, load balancing, and sparsity control in a single router, creating inherent gradient conflicts, thereby degrading model capacity (Section~\ref{sec:perf_comparison}).
}

\subsection{BEAM: Binary Expert Activation Masking}

\begin{figure}[!htp]
    \centering
    \includegraphics[width=\linewidth]{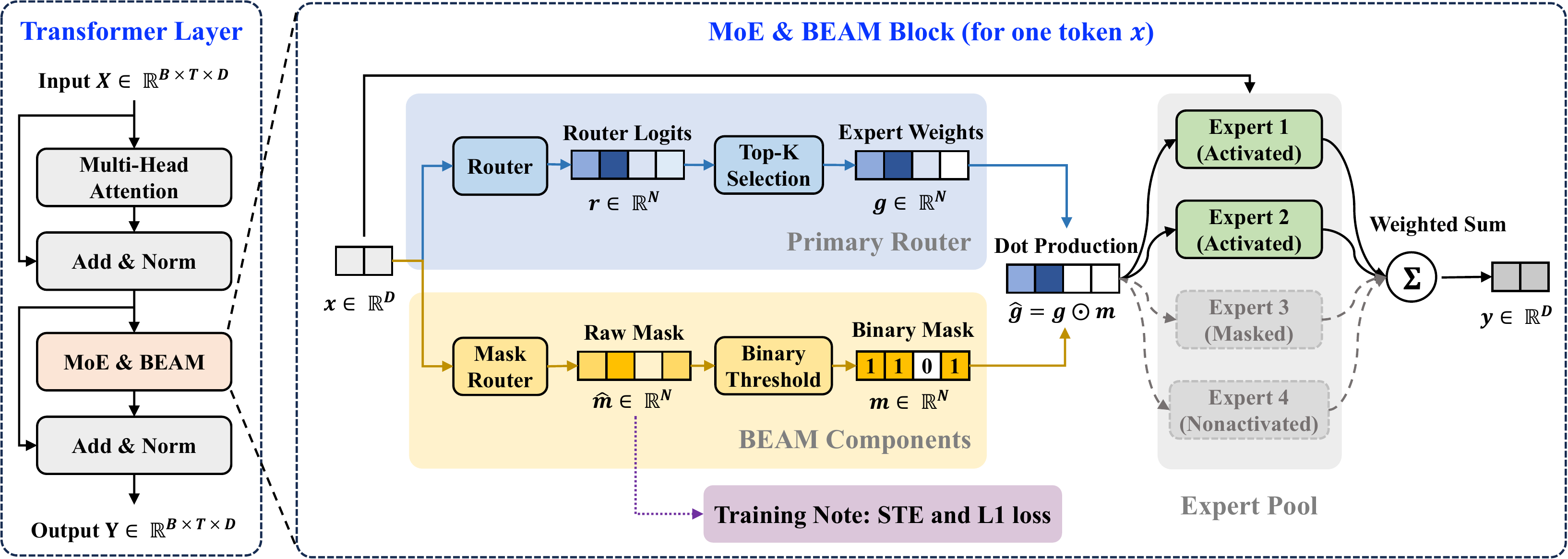}
    \caption{The illustration of our proposed BEAM method with 4 experts and K=3 as an example. 
    }
    \label{fig:pipeline}
\end{figure}

The above limitations motivate BEAM, which enables token-adaptive expert activation by introducing a lightweight and learnable mask router that generates a binary mask to selectively deactivate redundant experts from the standard Top-K candidate set, as shown in Figure~\ref{fig:pipeline}. Formally, given an input token embedding $\mathbf{x} \in \mathbb{R}^{d_h}$, BEAM operates in four steps.

\noindent \textbf{Step 1: Standard Top-K Routing.}  
The primary router $\mathcal{R}$ computes logits $\mathbf{r} = \mathcal{R}(\mathbf{x}) \in \mathbb{R}^N$, where $N$ is the total number of experts. The $\mathrm{Top}\text{-}K(\cdot)$ operator retains the K largest values and sets the rest to $-\infty$. The normalized routing weights are computed as:
\begin{equation}
    \mathbf{g}_i = \mathrm{Softmax}(\mathrm{Top}\text{-}K(\mathbf{r}))_i, \quad i=1,\dots,N,
\end{equation}
where $\mathbf{g}_i > 0$ only for the top K experts and $\sum_{i=1}^N \mathbf{g}_i = 1$.

\noindent \textbf{Step 2: Raw Mask Generation.}  
A lightweight auxiliary mask router, parameterized by $\mathbf{W}_m \in \mathbb{R}^{d_h \times N}$, processes the same input token $\mathbf{x}$ to generate a raw mask $\hat{\mathbf{m}}$. We apply a Sigmoid activation $\sigma$ to constrain the raw mask values to the range $(0, 1)$:
\begin{equation}
\hat{\mathbf{m}} = \sigma(\mathbf{x} \mathbf{W}_m).
\end{equation}
$\hat{\mathbf{m}}$ reflects the model's confidence in the necessity of each expert for the current token.

\noindent \textbf{Step 3: Binary Masking.}  
We binarize the raw mask $\hat{\mathbf{m}}$ using a fixed threshold of $\tau=0.5$ to obtain a discrete mask $\mathbf{m} \in \{0, 1\}^N$:
\begin{equation}
    \mathbf{m}_i = 
    \begin{cases}
        1, & \text{if } \hat{\mathbf{m}}_i \geq 0.5, \\
        0, & \text{otherwise},
    \end{cases}.
    \label{eq:binary_threshold}
\end{equation}
Since $\mathbf{m}_i = 0$ disables expert $i$ regardless of its Top-K status, the number of activated experts per token can be possibly reduced to $0$.

\noindent \textbf{Step 4: Masked Aggregation.}  
The final routing weights $\hat{\mathbf{g}}$ are obtained by performing an element-wise multiplication between the Top-K weights $\mathbf{g}$ and the binary mask $\mathbf{m}$:
\begin{equation}
    \hat{\mathbf{g}} = \mathbf{g} \odot \mathbf{m},
\end{equation}
and the layer output is computed by aggregating the masked activations:
\begin{equation}
\mathbf{y} = \sum_{i=1}^N \hat{\mathbf{g}}_i \cdot \mathcal{E}_i(\mathbf{x}).
\end{equation}

This design provides three key advantages. First, it decouples routing and sparsification, {\em i.e.}, the primary router handles expert selection and load balancing, while the mask router focuses exclusively on redundancy elimination, avoiding conflicting optimization objectives. Second, expert sparsity is learned end-to-end without manual tuning, enabling aggressive expert reduction while preserving model capability. Third, the binary mask provides a hardware-friendly signal that can be directly leveraged by custom CUDA kernels, facilitating efficient real-world deployment.

\subsection{Training Strategy}
BEAM is trained end-to-end using two key components. The first is the Straight-Through Estimator (STE) to handle the non-differentiable binarization operation. The second is an auxiliary sparsity regularization loss added to the standard MoE objective to jointly optimize task performance, expert load balancing, and computational efficiency.

\subsubsection{Straight-Through Estimator}
The binary mask $\mathbf{m}$ is generated via a non-differentiable hard thresholding function as defined in Equation~\ref{eq:binary_threshold}. To enable the mask router to be trained via backpropagation, we adopt the STE method~\citep{bengio2013estimating} to approximate the gradient. Specifically, during the backward pass, the threshold function is treated as an identity mapping, allowing the gradient of the loss $\mathcal{L}$ with respect to $\mathbf{m}$ to be propagated directly to the raw mask $\hat{\mathbf{m}}$:
\begin{equation}
\frac{\partial \mathcal{L}}{\partial \hat{\mathbf{m}}} \approx \frac{\partial \mathcal{L}}{\partial \mathbf{m}}.
\end{equation}
This allows the mask router to be trained with backpropagation despite the discrete nature of $\mathbf{m}$. Note that all Top-K experts are computed regardless of $\mathbf{m}$ to ensure proper gradient flow during training.

To ensure stable training, we initialize the mask router parameters to zero. This yields $\hat{\mathbf{m}} = 0.5$ and $\mathbf{m} = 1$ for all experts at the start of training, which preserves the original Top-K behavior and allows sparsity to emerge gradually as training proceeds.

\subsubsection{Sparsity-Guided Optimization}
The total training loss combines three terms. In addition to the standard language modeling loss $\mathcal{L}_{lm}$ and the expert load-balancing loss $\mathcal{L}_{bal}$, we introduce an auxiliary sparsity regularization loss $\mathcal{L}_{reg}$, defined as the $L_1$ norm of the raw mask restricted to the Top-K candidate set $\mathcal{T}_K$:
\begin{equation}
\mathcal{L}_{reg} = \frac{1}{K} \sum_{i \in \mathcal{T}_K} |\hat{\mathbf{m}}_i|.     \end{equation}
$\mathcal{L}_{reg}$ directly encourages the mask router to suppress redundant experts among selected candidates without introducing extraneous gradients for non-selected experts.

The overall objective is a weighted sum:
\begin{equation}
\mathcal{L} = \mathcal{L}_{lm} + \alpha \mathcal{L}_{bal} + \beta \mathcal{L}_{reg},
\end{equation}
where $\alpha$ and $\beta$ are hyperparameters that control the balance between expert utilization and computational efficiency. Through this sparsity-guided optimization, BEAM learns to activate only the necessary experts for each token, achieving high inference speed without compromising performance.

\subsection{Theoretical Analysis}

We provide a theoretical analysis of BEAM’s training dynamics. The core operation is the masked routing weight $\hat{\mathbf{g}} = \mathbf{g} \odot \mathbf{m}$, where $\mathbf{g}$ is the output of the primary router and $\mathbf{m} = \mathbb{I}(\hat{\mathbf{m}} \geq 0.5)$ is the binary mask derived from $\hat{\mathbf{m}} = \sigma(\mathbf{a})$, with $\mathbf{a} = \mathbf{x} \mathbf{W}_m$ being the mask router pre-activation.

The mask router receives gradients from two sources: the task loss $\mathcal{L}_{lm}$ propagated through the masked routing weights $\hat{\mathbf{g}}$ via STE, and the sparsity regularization $\mathcal{L}_{reg}$ applied directly to the Top-K mask values. The load-balancing loss $\mathcal{L}_{bal}$ is computed solely from the primary router's weights $\mathbf{g}$ before masking and does not produce gradients for the mask router.

\begin{definition}[Gradient for Mask Router]
Under STE, the full gradient of $\mathcal{L}$ with respect to the mask router pre-activation $\mathbf{a}$ is:
\begin{equation}
    \left(\nabla_{\mathbf{a}} \mathcal{L}\right)_i
    = \left( \frac{\partial \mathcal{L}_{lm}}{\partial \hat{g}_i} \cdot \mathbf{g}_i + \frac{\beta}{K} \cdot \mathbf{1}_{[i \in \mathcal{T}_K]} \right) \sigma'(a_i),
\label{eq:ste_grad_beam}
\end{equation}
where $\mathcal{T}_K$ denotes the Top-K candidate set and $\mathbf{1}_{[i \in \mathcal{T}_K]}$ is its indicator function.
\end{definition}

\begin{theorem}[Selective Gradient Propagation]
\label{thm:selective_grad}
The gradient in Equation~\ref{eq:ste_grad_beam} satisfies:
\begin{align}
    &\mathbf{g}_i = 0 \quad \Rightarrow \quad \left(\nabla_{\mathbf{a}} \mathcal{L}\right)_i = 0, \\
    &\mathbf{g}_i > 0 \quad \Rightarrow \quad \left(\nabla_{\mathbf{a}} \mathcal{L}\right)_i = \left(\frac{\partial \mathcal{L}_{lm}}{\partial \hat{g}_i} \cdot \mathbf{g}_i + \frac{\beta}{K}\right) \sigma'(a_i).
\end{align}
\end{theorem}

\begin{proof}
Since $\hat{\mathbf{m}}_i = \sigma(a_i) \in (0,1)$, the L1 gradient simplifies to $\partial |\hat{\mathbf{m}}_i|/\partial \hat{\mathbf{m}}_i = 1$. Note that $\sigma'(a_i) > 0$ for all $a_i \in \mathbb{R}$.

\noindent \textbf{Case 1:} If $\mathbf{g}_i = 0$, then $i \notin \mathcal{T}_K$. The task-loss term vanishes because $\mathbf{g}_i = 0$, and the regularization term vanishes because $\mathcal{L}_{reg}$ is restricted to $\mathcal{T}_K$. Hence $\left(\nabla_{\mathbf{a}} \mathcal{L}\right)_i = 0$, and the mask router receives no learning signal for non-selected experts.

\noindent \textbf{Case 2:} If $\mathbf{g}_i > 0$, then $i \in \mathcal{T}_K$ and both terms contribute. The gradient direction is determined by the sign of $\frac{\partial \mathcal{L}_{lm}}{\partial \hat{g}_i} \cdot \mathbf{g}_i + \frac{\beta}{K}$: the task-loss term drives $a_i$ toward values that reduce $\mathcal{L}_{lm}$, while the constant $\frac{\beta}{K}$ consistently pushes $a_i$ downward to encourage sparsity. Expert $i$ is retained when its task contribution outweighs the sparsity pressure ($\frac{\partial \mathcal{L}_{lm}}{\partial \hat{g}_i} \cdot \mathbf{g}_i < -\frac{\beta}{K}$), and pruned otherwise. The hyperparameter $\beta$ directly controls this trade-off.

\end{proof}

Further analysis of full expert masking behaviour and details of the efficient BEAM implementation in vLLM are provided in Appendix~\ref{sec:append_full_mask_analysis} and Appendix~\ref{sec:append_cuda_kernel}, respectively.

\section{Experiments}

\vspace{-2mm}

\subsection{Experimental Setup}

\vspace{-2mm}

\noindent\textbf{Models and Training Data~}
We evaluate BEAM on three representative MoE models: Qwen1.5‑MoE‑A2.7B \citep{bai2023qwen}, DeepSeekV2‑Lite \citep{liu2024deepseek}, and Qwen3‑30B‑A3B \citep{yang2025qwen3}. We conduct supervised fine-tuning using the Tulu 3 SFT Mixture Dataset~\citep{lambert2024tulu}, which covers reasoning, coding, and general knowledge tasks. \rebuttal{All baselines and BEAM are fine-tuned on the same dataset with identical training configurations to ensure fair comparison.}

\noindent\textbf{Baselines~}
We compare against five methods:
(1) \rebuttal{\textbf{Top-K Reduced}} trains with a smaller Top-K.
(2) \textbf{Top-K Pruning} \rebuttal{trains with the original Top-K and reduces Top-K at inference.}
(3) \textbf{MoE-Dynamic}~\citep{huang2024harder} activates experts until cumulative routing probability exceeds threshold $\phi$.
(4) \textbf{AdaMoE}~\citep{zeng2024adamoe} adds null experts with zero computation.
\rebuttal{(5) \textbf{DynMoE}~\citep{guo2024dynamic} uses sigmoid router to adaptively determine activated experts.}

\noindent\textbf{Evaluation Benchmarks~}
For accuracy evaluation, we use eight benchmarks from OpenCompass~\citep{2023opencompass} across three domains: \texttt{\textbf{Reasoning}} (Math~\citep{hendrycksmath2021}, GSM8K~\citep{cobbe2021gsm8k}), HumanEval(H\_Eval)~\citep{chen2021evaluating}), \texttt{\textbf{Knowledge}} (MMLU~\citep{hendryckstest2021}, CEVAL~\citep{huang2023ceval}, CMMLU~\citep{li2023cmmlu}), and \texttt{\textbf{Common Sense}} (CommonsenseQA(CSQA)~\citep{talmor-etal-2019-commonsenseqa}, BoolQ~\citep{clark2019boolq}).

For acceleration evaluation, we report Time per Output Token (TPOT), Time to First Token (TTFT), and throughput under varying QPS using vLLM~\citep{kwon2023efficient}. All models run on a single NVIDIA H20 GPU with fixed input/output lengths of $128/32$ tokens and $5000$ test samples.

\noindent\textbf{Hyperparameters~}
\rebuttal{For MoE-Dynamic, AdaMoE, and BEAM, we tune their respective hyperparameters, {\em i.e.}, cumulative probability threshold $\phi$, null expert count, and L1 loss coefficient $\beta$, to match comparable sparsity levels with other methods at each setting.} All experiments are conducted on NVIDIA H20 GPUs under identical hyperparameter settings, as detailed in Appendix~\ref{sec:appendix_experiment_setup}.

\subsection{Performance Comparison}
\label{sec:perf_comparison}

\vspace{-4mm}

\begin{table}[!htbp]
  \centering
  \caption{Performance comparison on Qwen1.5-MoE-A2.7B under different sparsity levels. Best results within each sparsity group are marked in \textbf{bold}.}
  \label{tab:qwen1_5_performance}
  \resizebox{\linewidth}{!}{
  \begin{tabular}{l|c|ccc|ccc|cc|c}
  \toprule
  \rowcolor{gray!25}
   &
  & \multicolumn{3}{c|}{\textbf{Reasoning}}
  & \multicolumn{3}{c|}{\textbf{Knowledge}}
  & \multicolumn{2}{c|}{\textbf{CommonSense}}
  &   \\

  \noalign{\vspace{-0.85mm}}
  \cmidrule(lr){3-5}  \cmidrule(lr){6-8} \cmidrule(lr){9-10} \noalign{\vspace{-0.85mm}}
  \rowcolor{gray!25}
  \multirow{-2}{*}{\textbf{Methods \textbackslash Tasks}}  &
  \multirow{-2}{*}{\makecell[c]{\textbf{Avg. K} } }&
  \textbf{MATH} & \textbf{GSM8K} & \textbf{H\_Eval} & \textbf{MMLU} & \textbf{CEVAL} & \textbf{CMMLU} & \textbf{BoolQ} & \textbf{CSQA} & \multirow{-2}{*}{\makecell[c]{\textbf{Avg.} \\ (Acc. $\uparrow$)  } }  \\
  \midrule
  Qwen1.5-MoE$_{~\text{K=4}}$ & 4.00 & 23.04 & 57.47 & 50.61 & 59.28 & 74.15 & 75.18 & 72.63 & 81.33 & 61.71 \\
  \midrule
  \rowcolor{blue!5}
  \multicolumn{11}{c}{\textit{Mid Sparsity}} \\
  \rowcolor{gray!10}
  Top-K Pruning$_{~\text{K=2}}$ & 2.00 & 22.40 & 49.36 & 43.90 & 58.69 & \textbf{70.76} & 71.40 & 62.97 & 80.51 & 57.50 \\
  \rebuttal{Top-K Reduced}$_{~\text{K=2}}$ & 2.00 & 21.98 & 53.68 & 51.83 & \textbf{58.81} & 70.53 & \textbf{71.50} & 77.52 & 80.34 & 60.77 \\
  \rowcolor{gray!10}
  MoE-Dynamic$_{~\text{$\phi$=0.4}}$ & 2.20 & 20.94 & 51.86 & 47.56 & 57.85 & 67.35 & 67.84 & 73.82 & 80.34 & 58.45 \\
  AdaMoE$_{~\text{Null=60}}$ & 1.53 & 17.92 & 53.98 & 46.95 & 47.89 & 46.95 & 47.82 & 72.05 & 62.98 & 49.57 \\
  \rowcolor{gray!10}
  \textbf{BEAM}$_{~(\beta=0.01)}$ & 1.56 & \textbf{24.78} & \textbf{55.50} & \textbf{53.05} & 58.75 & 70.47 & 69.18 & \textbf{78.32} & \textbf{80.84} & \textbf{61.36} \\
  \midrule
  \rowcolor{blue!5}
  \multicolumn{11}{c}{\textit{High Sparsity}} \\
  \rowcolor{gray!10}
  Top-K Pruning$_{~\text{K=1}}$ & 1.00 & 9.78 & 34.12 & 25.61 & 53.84 & 60.62 & 59.46 & 52.45 & 74.12 & 46.25 \\
  \rebuttal{Top-K Reduced}$_{~\text{K=1}}$ & 1.00 & 18.82 & 49.20 & 41.46 & 53.97 & 60.70 & 60.78 & 72.26 & 76.41 & 54.20 \\
  \rowcolor{gray!10}
  MoE-Dynamic$_{~\text{$\phi$=0.2}}$ & 1.47 & 17.22 & 45.64 & 42.07 & 53.29 & 58.46 & 58.85 & \textbf{74.25} & 74.61 & 53.05 \\
  AdaMoE$_{~\text{Null=120}}$ & 1.26 & 15.76 & 47.38 & 42.32 & 46.55 & 41.11 & 43.57 & 64.71 & 64.37 & 45.72 \\
  \rowcolor{gray!10}
  \textbf{BEAM}$_{~(\beta=0.1)}$ & \textbf{0.56} & \textbf{23.54} & \textbf{55.04} & \textbf{49.39} & \textbf{58.05} & \textbf{70.22} & \textbf{67.52} & 72.69 & \textbf{79.77} & \textbf{59.53} \\
  \midrule
  \rowcolor{blue!5}
  \multicolumn{11}{c}{\textit{Extreme Sparsity}} \\
  \rowcolor{gray!10}
  \textbf{BEAM}$_{~(\beta=1.0)}$ & \textbf{0.11} & 18.72 & 51.18 & 42.07 & 54.17 & 58.97 & 57.15 & 69.11 & 69.94 & 52.66 \\
  \bottomrule
  \end{tabular}
  }
\end{table}

\begin{table}[!htbp]
    \centering
    \caption{Performance comparison on Qwen3-30B-A3B under different sparsity levels. Best results within each sparsity group are marked in \textbf{bold}.}
    \label{tab:qwen3_performance}
      \resizebox{\linewidth}{!}{
      \begin{tabular}{l|c|ccc|ccc|cc|c}
      \toprule
      \rowcolor{gray!25}
       &
      & \multicolumn{3}{c|}{\textbf{Reasoning}}
      & \multicolumn{3}{c|}{\textbf{Knowledge}}
      & \multicolumn{2}{c|}{\textbf{CommonSense}}
      &   \\

    \noalign{\vspace{-0.85mm}}
    \cmidrule(lr){3-5}  \cmidrule(lr){6-8} \cmidrule(lr){9-10} \noalign{\vspace{-0.85mm}}
    \rowcolor{gray!25}
    \multirow{-2}{*}{\textbf{Methods \textbackslash Tasks}}  &
    \multirow{-2}{*}{\makecell[c]{\textbf{Avg. K} } }&
    \textbf{MATH} & \textbf{GSM8K} & \textbf{H\_Eval} & \textbf{MMLU} & \textbf{CEVAL} & \textbf{CMMLU} & \textbf{BoolQ} & \textbf{CSQA} & \multirow{-2}{*}{\makecell[c]{\textbf{Avg.} \\ (Acc. $\uparrow$)  } }  \\
    \midrule
    Qwen3-30B-A3B$_{~\text{K=8}}$ & 8.00 & 58.28 & 88.02 & 82.93 & 81.80 & 83.56 & 83.69 & 86.76 & 86.24 & 81.41 \\
    \midrule
    \rowcolor{blue!5}
    \multicolumn{11}{c}{\textit{Mid Sparsity}} \\
    \rowcolor{gray!10}
    Top-K Pruning$_{~\text{K=4}}$ & 4.00 & 48.76 & 49.51 & 76.83 & 73.49 & 75.85 & 76.27 & 81.04 & 75.02 & 69.60 \\
    \rebuttal{Top-K Reduced}$_{~\text{K=4}}$ & 4.00 & \textbf{56.44} & 84.46 & 80.49 & 78.27 & 80.40 & 78.27 & 87.68 & 85.18 & 78.90 \\
    \rowcolor{gray!10}
    MoE-Dynamic$_{~\text{$\phi$=0.3}}$ & 5.04 & 54.28 & 82.03 & 76.22 & 77.86 & 78.17 & 78.93 & 87.22 & 85.18 & 77.49 \\
    AdaMoE$_{~\text{Null=128}}$ & 4.02 & 41.68 & 62.44 & 69.93 & 72.51 & 62.71 & 63.39 & 84.71 & 77.15 & 66.81 \\
    \rowcolor{gray!10}
    \textbf{BEAM}$_{~(\beta=0.01)}$ & 4.23 & 55.16 & \textbf{85.52} & \textbf{81.71} & \textbf{80.09} & \textbf{81.46} & \textbf{81.53} & \textbf{88.07} & \textbf{86.40} & \textbf{79.99} \\
    \midrule
    \rowcolor{blue!5}
    \multicolumn{11}{c}{\textit{High Sparsity}} \\
    \rowcolor{gray!10}
    Top-K Pruning$_{~\text{K=2}}$ & 2.00 & 0.68 & 1.14 & 0.00 & 25.39 & 17.51 & 16.42 & 17.31 & 16.87 & 11.92 \\
    \rebuttal{Top-K Reduced}$_{~\text{K=2}}$ & 2.00 & 44.60 & 77.79 & 72.56 & 72.99 & 72.63 & 72.99 & 83.85 & 80.02 & 72.18 \\
    \rowcolor{gray!10}
    MoE-Dynamic$_{~\text{$\phi$=0.1}}$ & 1.74 & 40.10 & 75.36 & 72.38 & 67.59 & 64.38 & 63.67 & 81.07 & 78.87 & 67.93 \\
    AdaMoE$_{~\text{Null=256}}$ & 2.64 & 34.06 & 73.01 & 58.54 & 44.50 & 38.75 & 37.76 & 61.04 & 60.44 & 51.01 \\
    \rowcolor{gray!10}
    \textbf{BEAM}$_{~(\beta=0.1)}$ & \textbf{1.23} & \textbf{55.44} & \textbf{85.90} & \textbf{81.10} & \textbf{76.06} & \textbf{74.04} & \textbf{74.54} & \textbf{85.75} & \textbf{84.28} & \textbf{77.14} \\
    \midrule
    \rowcolor{blue!5}
    \multicolumn{11}{c}{\textit{Extreme Sparsity}} \\
    \rowcolor{gray!10}
    \rebuttal{Top-K Reduced}$_{~\text{K=1}}$ & 1.00 & 27.96 & 63.08 & 51.22 & 58.26 & 52.97 & 53.05 & 56.21 & 68.88 & 53.95 \\
    \textbf{BEAM}$_{~(\beta=1.0)}$ & \textbf{0.56} & \textbf{52.20} & \textbf{81.96} & \textbf{77.44} & \textbf{69.44} & \textbf{66.28} & \textbf{69.36} & \textbf{78.47} & \textbf{80.10} & \textbf{71.91} \\
    \bottomrule
    \end{tabular}
    }
\end{table}

\begin{table}[!htbp]
    \centering
    \caption{Performance comparison on DeepSeekV2-Lite under different sparsity levels. Best results within each sparsity group are marked in \textbf{bold}.}
    \label{tab:deepseek_performance}
      \resizebox{\linewidth}{!}{
      \begin{tabular}{l|c|ccc|ccc|cc|c}
      \toprule
      \rowcolor{gray!25}
       &
      & \multicolumn{3}{c|}{\textbf{Reasoning}}
      & \multicolumn{3}{c|}{\textbf{Knowledge}}
      & \multicolumn{2}{c|}{\textbf{CommonSense}}
      &   \\

    \noalign{\vspace{-0.85mm}}
    \cmidrule(lr){3-5}  \cmidrule(lr){6-8} \cmidrule(lr){9-10} \noalign{\vspace{-0.85mm}}
    \rowcolor{gray!25}
    \multirow{-2}{*}{\textbf{Methods \textbackslash Tasks}}  &
    \multirow{-2}{*}{\makecell[c]{\textbf{Avg. K} } }&
    \textbf{MATH} & \textbf{GSM8K} & \textbf{H\_Eval} & \textbf{MMLU} & \textbf{CEVAL} & \textbf{CMMLU} & \textbf{BoolQ} & \textbf{CSQA} & \multirow{-2}{*}{\makecell[c]{\textbf{Avg.} \\ (Acc. $\uparrow$)  } }  \\
    \midrule
    DeepSeekV2-Lite$_{~\text{K=6}}$ & 6.00 & 20.02 & 62.70 & 43.90 & 55.04 & 55.26 & 60.93 & 75.20 & 68.14 & 55.15 \\
    \midrule
    \rowcolor{blue!5}
    \multicolumn{11}{c}{\textit{Mid Sparsity}} \\
    \rowcolor{gray!10}
    Top-K Pruning$_{~\text{K=4}}$ & 4.00 & 15.10 & 57.09 & 37.80 & 46.68 & 53.25 & 60.00 & 69.69 & 67.40 & 50.88 \\
    \rebuttal{Top-K Reduced}$_{~\text{K=4}}$ & 4.00 & 16.58 & 57.24 & 40.85 & 54.70 & 55.89 & 60.40 & 76.39 & \textbf{72.48} & 54.32 \\
    \rowcolor{gray!10}
    MoE-Dynamic$_{~\text{$\phi$=0.3}}$ & 4.31 & 19.08 & 35.63 & 38.35 & 42.95 & 53.36 & 58.12 & 63.55 & 70.60 & 47.70 \\
    AdaMoE$_{~\text{Null=64}}$ & 3.25 & 12.00 & 37.76 & 25.67 & 53.01 & \textbf{57.28} & 58.85 & 62.32 & 69.86 & 47.09 \\
    \rowcolor{gray!10}
    \textbf{BEAM}$_{~(\beta=0.01)}$ & 2.61 & \textbf{20.36} & \textbf{60.27} & \textbf{46.95} & \textbf{56.65} & 54.12 & \textbf{60.42} & \textbf{76.57} & 65.11 & \textbf{55.06} \\
    \midrule
    \rowcolor{blue!5}
    \multicolumn{11}{c}{\textit{High Sparsity}} \\
    \rowcolor{gray!10}
    Top-K Pruning$_{~\text{K=2}}$ & 2.00 & 13.38 & 46.02 & 28.66 & 43.25 & 51.06 & 36.91 & 58.90 & 58.64 & 42.10 \\
    \rebuttal{Top-K Reduced}$_{~\text{K=2}}$ & 2.00 & 15.18 & 51.10 & 34.76 & \textbf{51.65} & 53.35 & 60.39 & 68.87 & 66.83 & 50.27 \\
    \rowcolor{gray!10}
    MoE-Dynamic$_{~\text{$\phi$=0.1}}$ & 3.90 & 15.96 & 56.56 & 39.00 & 28.52 & 55.10 & 58.43 & 34.56 & \textbf{71.09} & 44.90 \\
    AdaMoE$_{~\text{Null=128}}$ & 2.11 & 9.24 & 28.96 & 20.73 & 39.70 & 54.83 & 52.29 & 54.50 & 70.27 & 41.31 \\
    \rowcolor{gray!10}
    \textbf{BEAM}$_{~(\beta=0.1)}$ & \textbf{1.08} & \textbf{17.18} & \textbf{59.21} & \textbf{43.29} & 48.08 & \textbf{56.15} & \textbf{60.67} & \textbf{71.07} & 70.93 & \textbf{53.32} \\
    \midrule
    \rowcolor{blue!5}
    \multicolumn{11}{c}{\textit{Extreme Sparsity}} \\
    \rowcolor{gray!10}
    \rebuttal{Top-K Reduced}$_{~\text{K=1}}$ & 1.00 & 7.90 & 31.99 & 25.00 & 28.75 & 41.46 & 45.01 & 53.14 & 52.91 & 35.77 \\
    \textbf{BEAM}$_{~(\beta=1.0)}$ & \textbf{0.48} & \textbf{11.72} & \textbf{45.19} & \textbf{42.07} & \textbf{38.33} & \textbf{50.11} & \textbf{54.48} & \textbf{69.66} & \textbf{67.57} & \textbf{47.39} \\
    \bottomrule
    \end{tabular}
    }
\end{table}

Table~\ref{tab:qwen1_5_performance}, Table~\ref{tab:qwen3_performance}, and Table~\ref{tab:deepseek_performance} summarize the performance and sparsity results of BEAM and baselines across multiple MoE models\rebuttal{, organized by mid, high, and extreme sparsity levels.} \rebuttal{We report} average activated experts per token (Avg-K) and downstream task scores. \rebuttal{Comparisons with DynMoE are provided in Appendix~\ref{sec:more_baseline}.}

\noindent\textbf{BEAM achieves extreme sparsity with minimal performance loss.~}
\rebuttal{BEAM consistently preserves over \textbf{98\%} of original accuracy at mid sparsity across all three models while reducing Avg-K by 47\%--61\%. At high sparsity, Avg-K drops to as low as \textbf{14\%} of the original (e.g., $0.56/4$ on Qwen1.5) with over 95\% accuracy retained.
The advantage of BEAM is most evident under extreme sparsity. On DeepSeekV2, BEAM ($K=0.48$) outperforms Top-K Reduced ($K=1$) by \textbf{32.49\%}, while on Qwen3 the margin reaches \textbf{33.29\%}.
On Qwen1.5, BEAM reaches Avg-K $=0.11$, indicating that most tokens completely bypass routed experts, while still retaining 85\% of the original performance, which demonstrates effective token-adaptive redundancy removal.
}

\noindent\textbf{Existing dynamic routing methods underperform in post-training settings.~}
\rebuttal{Top-K Pruning degrades sharply at higher sparsity, while Top-K Reduced is more stable but its fixed per-token budget consistently underperforms BEAM. Even at extreme sparsity, BEAM with fewer average experts outperforms Top-K Reduced ($K=1$) on both Qwen3 and DeepSeek.
MoE-Dynamic and AdaMoE also fall short: the former requires model-specific threshold tuning without competitive trade-offs, while the latter suffers from performance degradation due to null-expert interference.
BEAM avoids these issues by decoupling sparsification from expert selection via a lightweight mask router, enabling stable training and preserving the original expert load balance (Appendix~\ref{sec:expert_load_balance_analysis}).}

\rebuttal{\noindent\textbf{$\beta$ provides smooth control over the sparsity--accuracy trade-off.~}
Increasing $\beta$ consistently improves sparsity with gradual accuracy loss (Tables~\ref{tab:qwen1_5_performance}--\ref{tab:deepseek_performance}), making it straightforward to adapt the method to deployment constraints via a single parameter. At $\beta = 0.1$, BEAM preserves over 95\% accuracy across all models, offering a good trade-off.}

\subsection{Acceleration Comparison}

We evaluate inference acceleration under both online and offline settings. In the online setting, models are deployed as services and we measure TTFT and TPOT across varying QPS to simulate real-world serving. In the offline setting, we use a large fixed batch size to maximize GPU utilization and report throughput, reflecting scenarios like large-scale LLM knowledge distillation. For fair comparison with performance-efficiency tradeoff, we tested the inference speed of all baseline methods and BEAM under \textit{High Sparsity}.

As shown in Figure~\ref{fig:model_accelerate_compare}, BEAM achieves consistent speedups across all models and settings. It achieves at least $1.3\times$ improvement in TPOT and over $1.1\times$ gains in both TTFT and throughput. Notably, on DeepSeek-V2-Lite at QPS=24, BEAM reaches up to $2.5\times$ decoding acceleration. 
The achievable speedup is limited by model architecture. For example, Qwen1.5-MoE-A2.7B contains 4 shared experts out of 8 total, limiting their MoE layer FLOPs reduction to at most 50\%. In contrast, Qwen3-30B-A3B has no shared experts, enabling an 85\% FLOPs reduction and substantially higher throughput gains.
In comparison, MoE-Dynamic and AdaMoE achieve limited sparsity and introduce extra overhead, yielding negligible or no acceleration benefits.

\begin{figure}[!htbp]
    \centering
    \includegraphics[width=\linewidth]{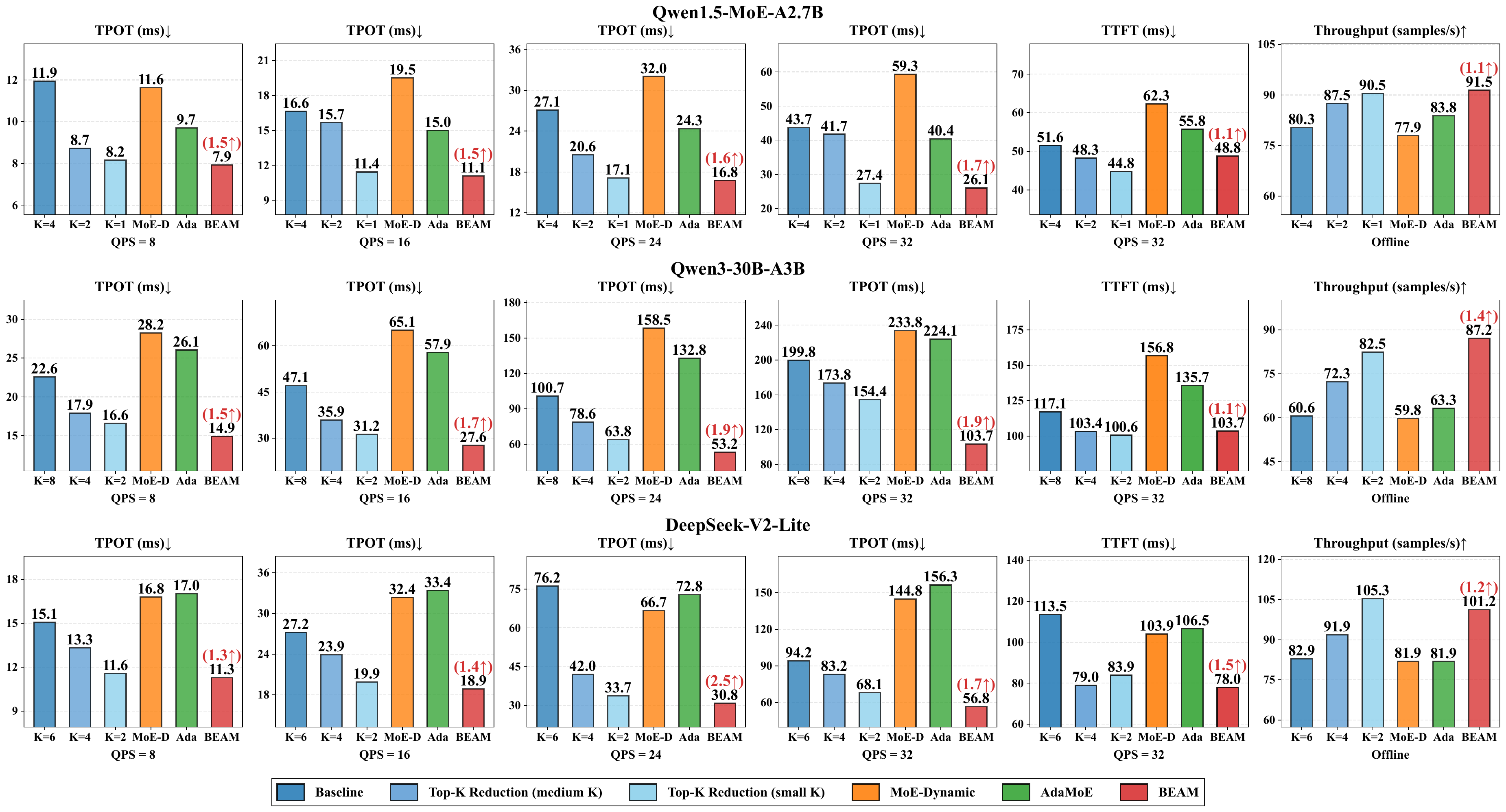}
    \caption{Comparison of TPOT, TTFT, and throughput across different methods.}
    \vspace{-4mm}
    \label{fig:model_accelerate_compare}
\end{figure}



\subsection{Ablation Study}
\vspace{-4mm}


\begin{table}[!htbp]
\centering
\begin{minipage}[t]{0.48\linewidth}
  \centering
  \captionsetup{font=small}
  \setlength{\tabcolsep}{5pt}
  \captionof{table}{Binary threshold $\tau$ ablation on Qwen1.5.}
  \label{tab:ablation_sigmoid}
  \renewcommand{\arraystretch}{1.2}
  \resizebox{\linewidth}{!}{
  \begin{tabular}{cccccc}
  \toprule
  \rowcolor{gray!25}
  \textbf{Threshold}
  & \textbf{Avg. K}
  & \textbf{Reason.}
  & \textbf{Know.}
  & \textbf{Common.} & \textbf{Avg. $\uparrow$} \\
  \midrule
  $\tau=0.1$ & 0.78 & 43.91 & \textbf{66.80} & 66.43 & 58.12 \\
  \rowcolor{gray!10}
  $\tau=0.3$ & 0.70 & \textbf{44.41} & 66.37 & 70.82 & 59.25 \\
  $\tau=0.5$ & 0.56 & 43.01 & 65.12 & \textbf{76.23} & \textbf{59.61} \\
  \rowcolor{gray!10}
  $\tau=0.7$ & 0.42 & 41.92 & 62.23 & 69.72 & 56.49 \\
  $\tau=0.9$ & \textbf{0.28} & 41.09 & 60.00 & 59.26 & 52.72 \\
  \bottomrule
  \end{tabular}
  }
\end{minipage}
\hfill
\begin{minipage}[t]{0.48\linewidth}
  \centering
  \captionsetup{font=small}
  \setlength{\tabcolsep}{3.5pt}
  \captionof{table}{Training configuration ablation on Qwen3.}
  \label{tab:ablation_training}
  \renewcommand{\arraystretch}{1.4}
  \resizebox{\linewidth}{!}{
  \begin{tabular}{lccccc}
  \toprule
  \rowcolor{gray!25}
  \textbf{Configs.}
  & \textbf{Avg. K}
  & \textbf{Reason.}
  & \textbf{Know.}
  & \textbf{Common.} & \textbf{Avg. $\uparrow$} \\
  \midrule
  \textbf{BEAM} & 1.23 & 74.15 & 74.88 & 85.02 & 77.14 \\
  \rowcolor{gray!10}
  \textbf{~- w/o. $\mathcal{L}_{reg}$} & 6.31 & 75.42 & 76.55 & 83.23 & 77.80~\textcolor{darkgreen}{(0.9\%$\uparrow$)} \\
  \textbf{~- $\mathcal{L}_1$ to $\mathcal{L}_2$} & 2.01 & 71.34 & 73.21 & 84.29 & 75.28~\textcolor{red}{(2.4\%$\downarrow$)} \\
  \rowcolor{gray!10}
  \textbf{~- Soft } & 1.34 & 12.70 & 21.94 & 42.30 & 23.56~\textcolor{red}{(69.5\%$\downarrow$)} \\
  \rebuttal{\textbf{~- Soft w/. Temp.}} & \rebuttal{1.78} & \rebuttal{65.30} & \rebuttal{72.35} & \rebuttal{82.29} & \rebuttal{73.31}~\textcolor{red}{(5.0\%$\downarrow$)} \\
  \bottomrule
  \end{tabular}
  }
\end{minipage}
\end{table}

\noindent\textbf{Ablation on Binary Threshold ~} 
We evaluate binarization thresholds $\tau \in \{0.1, 0.3, 0.5, 0.7, 0.9\}$ on Qwen1.5-MoE-A2.7B, as shown in Table~\ref{tab:ablation_sigmoid}.
Increasing $\tau$ monotonically reduces Avg-K and thus increases sparsity. 
We find that $\tau = 0.5$ achieves the best overall performance, largely driven by stronger commonsense results. A plausible explanation is that $\tau = 0.5$ offers the greatest gradient sensitivity around the decision boundary while maintaining a stable Top-K initialization.
Based on this result, we fix $\tau = 0.5$ and vary only the regularization coefficient $\beta$ to control sparsity.

\noindent\textbf{Ablation on Training Approach ~}
We evaluate several training variants on Qwen3-30B-A3B, including removing $\mathcal{L}_{\text{reg}}$, replacing L1 with L2 regularization, and replacing STE-based binary masking with soft-mask training. For the latter, we consider both plain sigmoid gating (Soft) and a temperature-scaled sigmoid that gradually sharpens the mask (Soft w/. Temp.). As shown in Table~\ref{tab:ablation_training}, removing $\mathcal{L}_{\text{reg}}$ slightly improves reasoning performance but substantially increases expert activation. L2 regularization is inferior to L1 in both sparsity and accuracy. Both soft-mask variants also underperform binary-mask training, where plain sigmoid gating fails severely because of the train-inference mismatch, while temperature scaling only partially mitigates this issue. Overall, the results support the use of $\mathcal{L}_{\text{reg}}$, L1 regularization, and STE-based binary masking in BEAM.

\section{Analysis}
\label{sec:analysis}
\vspace{-2mm}

\subsection{Token-wise Sparsity Analysis}

\begin{figure}[!htp]
    \vspace{-3mm}
    \centering
    \includegraphics[width=\linewidth]{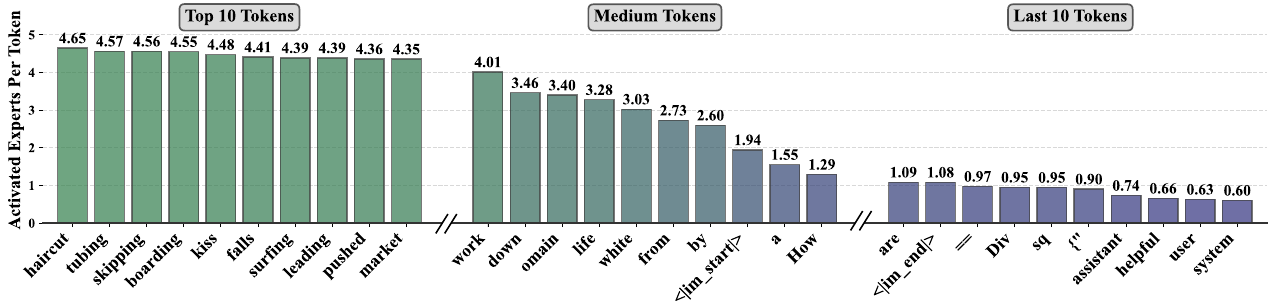}
    \caption{The average number of activated experts per token in BEAM: Qwen3-30B-A3B.}
    \label{fig:avg_exp_per_token}
    \vspace{-2mm}
\end{figure}

To understand how BEAM adapts computation per token, we visualize the average number of activated experts across tokens, as shown in Figure~\ref{fig:avg_exp_per_token}.
We obtain several key findings.
\textbf{1. Expert activation varies across tokens.} The most demanding tokens activate up to $4.65$ experts on average, versus only $0.6$ for the least demanding.
\textbf{2. Activation aligns with semantic richness.} Content words (e.g., nouns, verbs) consistently trigger more experts than function words (e.g., prepositions) and punctuation.
\textbf{3. Chat template tokens are highly redundant.} Fixed prompts like ``You are a helpful assistant'' activate few experts yet maintain performance, suggesting minimal informational value.
These findings show that BEAM dynamically allocates computation based on token informativeness.

\vspace{-3mm}
\subsection{\rebuttal{Layer-wise and Position-wise Analysis}}
\label{sec:position_masking}

\vspace{-2mm}

\begin{figure}[!h]
\centering
\begin{subfigure}[t]{0.42\linewidth}
    \centering
    \includegraphics[width=\linewidth]{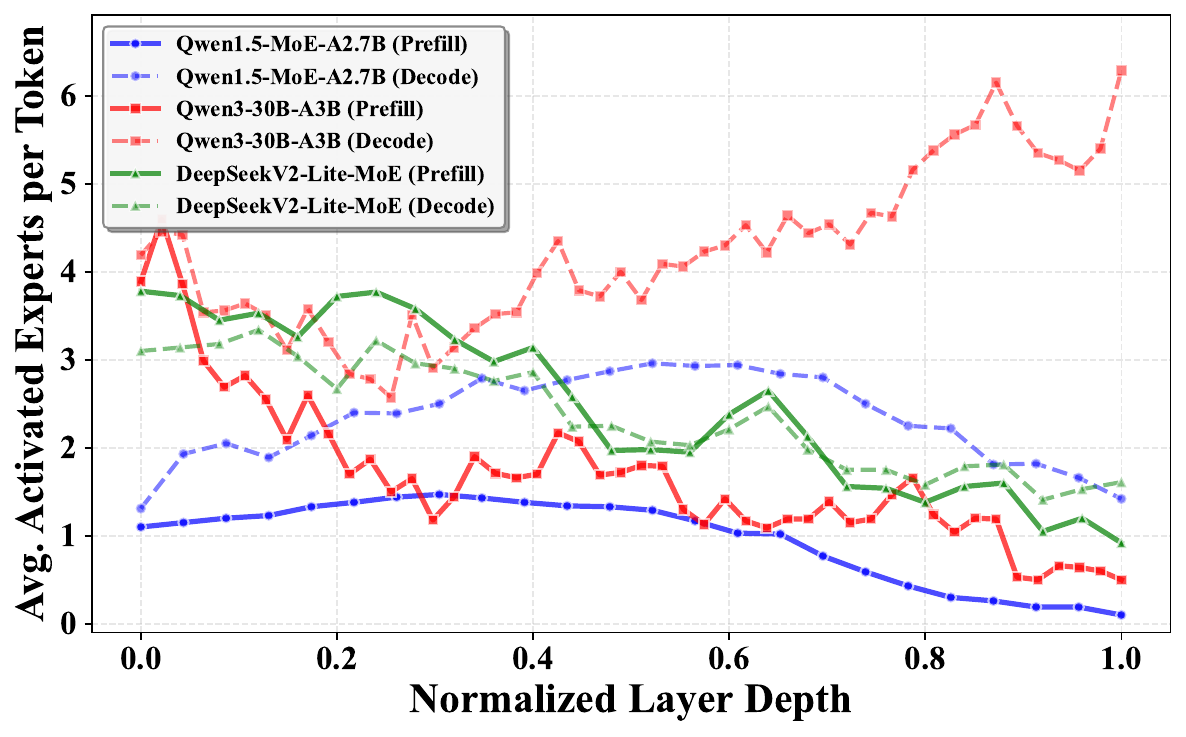}
    \vspace{-5mm}
    \caption{Layer-wise activated experts.}
    \label{fig:layer_sparsity}
\end{subfigure}
\hfill
\begin{subfigure}[t]{0.56\linewidth}
    \centering
    \rebuttal{\includegraphics[width=\linewidth]{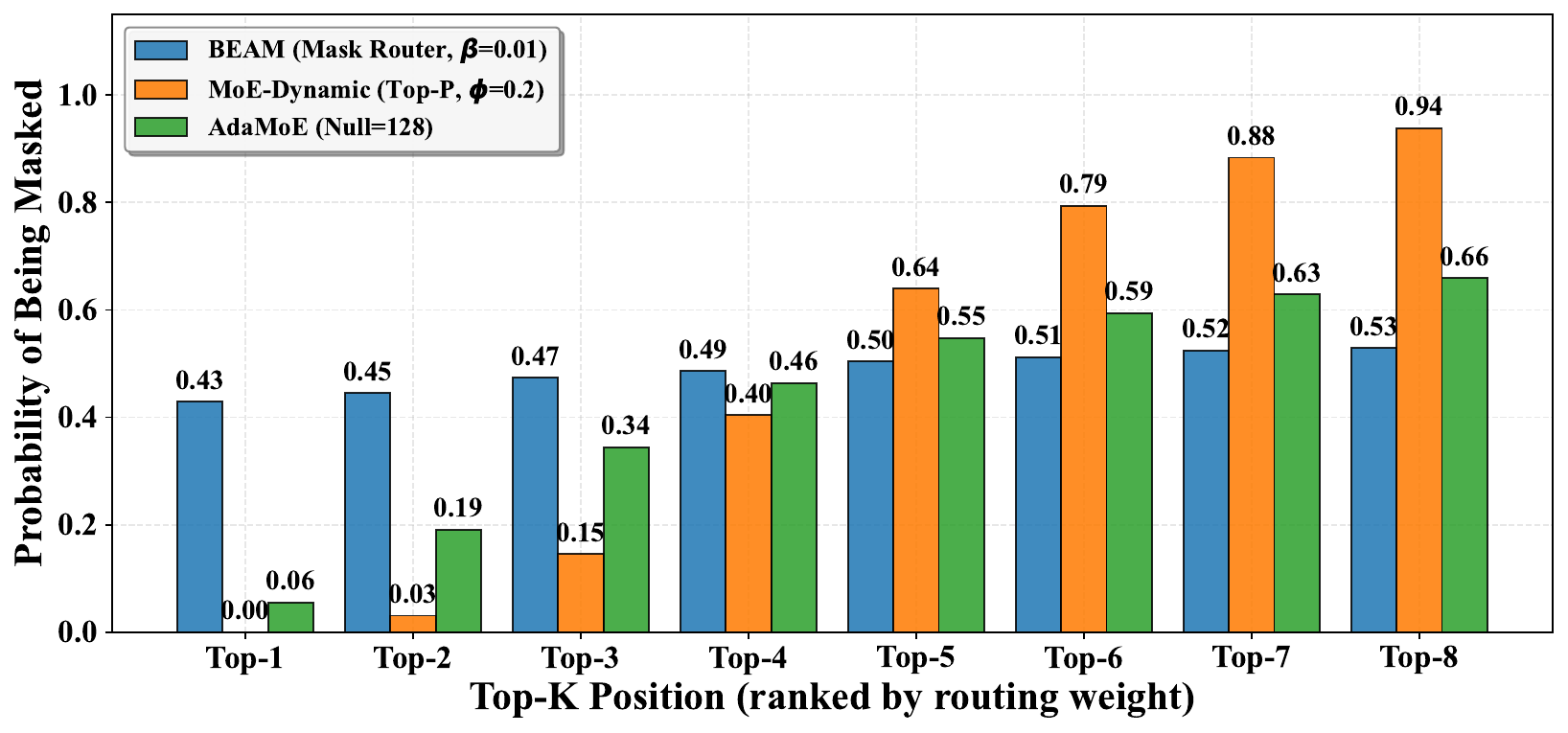}}
    \vspace{-5mm}
    \caption{\rebuttal{Position-wise masking probability.}}
    \label{fig:topk_wise_sparsity}
\end{subfigure}
\vspace{-1mm}
\caption{\rebuttal{Layer-wise sparsity and position-wise masking analysis.}}
\vspace{-3mm}
\label{fig:analysis_combined}
\end{figure}

We measure the average number of activated experts per layer during prefill and decode on 1,000 randomly sampled inputs (Figure~\ref{fig:layer_sparsity}). DeepSeek exhibits nearly identical expert usage in both phases, while Qwen1.5 and Qwen3 consistently use more experts during decoding. 
These models also develop an encoder-decoder-like pattern: shallower layers primarily support knowledge storage, while deeper layers allocate more expert capacity to decoding and reasoning, meaning that BEAM adapts layer-wise sparsity to functional roles.
\rebuttal{We further compare the per-position masking probability of BEAM, MoE-Dynamic, and AdaMoE on Qwen3 Under similar sparsity conditions. (Figure~\ref{fig:topk_wise_sparsity}). MoE-Dynamic exhibits strong position bias, never masking Top-1 (0.00) and applying extremely high masking beyond Top-6 (0.79$\sim$0.94). AdaMoE shows a monotonic increase from Top-1 (0.06) to Top-8 (0.66), indicating moderate but still rank-dependent bias. In contrast, BEAM shows only a mild increase from Top-1 (0.43) to Top-8 (0.53), demonstrating that it evaluates expert relevance based on token-specific features rather than routing rank. Another layer-wise rank masking analysis is provided in Appendix~\ref{sec:layerwise_rank}.}

We also provide task-specific acceleration analysis, expert load balancing analysis, and token-layer sparsity visualizations. Details can be found in Appendix~\ref{sec:task_specific_analysis}, ~\ref{sec:expert_load_balance_analysis}, and~\ref{sec:token_layer_sparsity_visualization}.

\vspace{-4mm}

\section{Conclusion}

\vspace{-3mm}

\rebuttal{We propose BEAM, a plug-and-play dynamic routing framework that introduces a lightweight mask router to selectively deactivate redundant experts within the Top-K set, enabling token-adaptive sparsity without modifying the model architecture. Integrated into vLLM via an efficient CUDA kernel, BEAM delivers up to 2.5$\times$ faster decoding and 1.4$\times$ higher throughput with over 98\% accuracy retention, demonstrating that decoupling sparsity control from routing enables stable and practical MoE inference acceleration.}




\bibliographystyle{plainnat}
\bibliography{example_paper}


\newpage

\newpage
\appendix

\section{Appendix on Method}

\subsection{Impact Statements}

This work focuses on improving the computational efficiency of Mixture-of-Experts models. We do not identify any societal impacts specific to the proposed method beyond those already associated with the general use and deployment of large language models.

\subsection{Limitations}
\label{sec:limitations}

\rebuttal{Our work has several limitations.
First, BEAM is evaluated on three MoE architectures; its effectiveness on other MoE designs (e.g., with different gating mechanisms or expert granularities) remains to be validated.
Second, BEAM requires a post-training SFT phase to learn the mask router, which incurs additional training cost proportional to the model size.
Third, the achievable inference speedup depends on the model's shared-expert ratio. For example, architectures with a large proportion of shared experts may benefit less, as shared-expert computation cannot be reduced by BEAM.
Finally, our acceleration benchmarks are conducted on single-GPU settings, and the interaction between BEAM's dynamic sparsity and multi-GPU expert parallelism strategies needs further investigation.}

\subsection{Behaviors under Zero Activation}
\label{sec:append_full_mask_analysis}

\rebuttal{Since BEAM permits between $0$ and $K$ activated experts, the zero-activation case requires special consideration. In a modern Transformer MoE block, the hidden state update follows:
\begin{equation}
\mathbf{h}' = \mathbf{h}
+ \sum\nolimits_{i=1}^{N} \hat{\mathbf{g}}_i \,\mathcal{E}_i(\mathcal{N}(\mathbf{h}))
+ \delta_{\mathrm{sh}}\,\mathbf{g}_{\mathrm{sh}} \,\mathcal{E}_{\mathrm{sh}}(\mathcal{N}(\mathbf{h})),
\end{equation}
where $\mathcal{N}(\cdot)$ denotes the normalization function, $\mathcal{E}_i(\cdot)$ denotes the $i$-th routed expert, $\hat{\mathbf{g}}_i \in \mathbb{R}$ denotes the normalized routing weight assigned to expert $i$, $\mathcal{E}_{\mathrm{sh}}(\cdot)$ denotes the shared expert, $\mathbf{g}_{\mathrm{sh}}$ denotes the routing weight assigned to the shared expert, and $\delta_{\mathrm{sh}} \in \{0,1\}$ is an indicator variable specifying whether a shared expert is present. 
When all routed experts are skipped, {\em i.e.}, $\hat{\mathbf{g}}_i = 0$ for all $i \in \{1,\dots,N\}$, then the update becomes:
\begin{equation}
\mathbf{h}' = \mathbf{h}
+ \delta_{\mathrm{sh}}\,\mathbf{g}_{\mathrm{sh}} \,\mathcal{E}_{\mathrm{sh}}(\mathcal{N}(\mathbf{h})).
\end{equation}
Therefore, if $\delta_{\mathrm{sh}}=1$, as in architectures with shared experts such as Qwen1.5-MoE and DeepSeek, the layer reduces to shared-expert-only computation. If $\delta_{\mathrm{sh}}=0$, as in architectures without shared experts such as Qwen3-MoE, the token bypasses the entire MoE layer through the residual path. 
}

\rebuttal{
Under zero activation, BEAM degenerates into a form of dynamic layer skipping, which has also been studied in prior work as an efficient inference acceleration mechanism~\citep{yang2025dash,lawson2025learning,amer2026conflayers}. Empirically, as shown in Section~\ref{sec:perf_comparison}, BEAM maintains strong model performance even when the average number of activated experts is below $1$, implying that zero-activation cases occur frequently in practice. This observation suggests that substantial layer computation in LLMs is redundant. Moreover, the analysis in Section~\ref{sec:analysis} shows that zero activation arises more often in deeper layers during prefill and for tokens with limited semantic content.
}

\subsection{Key Modifications for BEAM}
\label{sec:append_cuda_kernel}

We implement BEAM in vLLM by minimally extending its standard MoE CUDA pipeline with two kernel-level changes. First, \texttt{mask\_route\_kernel} writes the expert index as $-1$ whenever the corresponding mask logit is non-positive. Second, \texttt{moe\_align\_block\_size\_kernel} ignores all $-1$ entries during expert-wise token grouping and block alignment, thereby removing masked experts from subsequent computation. This modification is lightweight, preserves compatibility with vLLM’s existing optimizations such as operator fusion and memory coalescing, and introduces negligible integration overhead. The core code is provided below, and the full implementation will be released upon acceptance.


\begin{lstlisting}[language=C++]
    template<typename scalar_t>
    __global__ void mask_route_kernel(
        const int64_t* __restrict__ topk_ids,        // Input: Top-K expert indices
        const scalar_t* __restrict__ mask_logits,    // Input: Expert mask logits
        int64_t* __restrict__ output_ids,            // Output: Masked expert IDs
        const int num_tokens, const int top_k, const int num_experts) {
    
        int idx = blockIdx.x * blockDim.x + threadIdx.x;
        if (idx >= num_tokens * top_k) return;
        int token_idx = idx / top_k;
        int slot_idx = idx % top_k;
        int input_idx = token_idx * top_k + slot_idx;
        int64_t original_expert = topk_ids[input_idx];
    
        if (token_idx >= num_tokens || slot_idx >= top_k ||
            original_expert < 0 || original_expert >= num_experts) {
            output_ids[input_idx] = -1;
            return;
        }
    
        // Apply binary mask: keep if mask_logits > 0, else set to -1
        int expert_idx = token_idx * num_experts + original_expert;
        scalar_t logit = mask_logits[expert_idx];
        output_ids[input_idx] = (logit > 0) ? original_expert : -1;
    }
    
    template <typename scalar_t, typename token_cnts_t>
    __global__ void moe_align_block_size_kernel(/* ... */) {
        // ... thread setup ...
        for (int i = start_idx; i < end_idx; ++i) {
            int64_t expert_id = topk_ids[i];
            if (expert_id != -1) {  // Skip masked experts
                ++tokens_cnts[index(num_experts, threadIdx.x + 1, expert_id)];
            }
        }
        // ... rest of alignment logic ...
    }
\end{lstlisting}

\section{Appendix on Experiment}

\subsection{Experimental Setup}
\label{sec:appendix_experiment_setup}

\subsubsection{Models}

\begin{table}[!htp]
    \centering
    \caption{Main hyperparameters for each model.}
    \resizebox{\linewidth}{!}{
    \begin{tabular}{lccc}
        \toprule
        \rowcolor{gray!25}
        \textbf{Model Config} & \textbf{Qwen1.5-MoE-A2.7B} & \textbf{DeepSeekV2-Lite} & \textbf{Qwen3-30B-A3B} \\
        \midrule
        Total Params (B) & 14.3 & 16 & 30 \\
        Activated Params (B) & 2.7 & 2.4 & 3 \\
        MoE Layers / Total Layers & 24/24 & 26/27 & 48/48 \\
        Experts per MoE Layer & 60 & 64 & 128 \\
        Activated Experts per Token & 4 (selected) + 4 (shared) & 6 (selected) + 2 (shared) & 8 \\
        hidden size & 2560 & 2048 & 2048 \\
        intermediate size & 5632 & 10944 & 6144 \\
        Vocabulary Size & 151936 & 102400 & 151936 \\
        \bottomrule
        \toprule
        \rowcolor{gray!25}
        \textbf{Inference Setting} & \textbf{Qwen1.5-MoE-A2.7B} & \textbf{DeepSeekV2-Lite} & \textbf{Qwen3-30B-A3B} \\
        \midrule
        Temperature & 0.7 & 0.3 & 0.7 \\
        Top-$p$ & 0.8 & 0.95 & 0.8 \\
        Top-$k$ & 20 & 50 & 20 \\
        Repetition Penalty & 1.05 & 1.00 & 1.00 \\
        Max Output Tokens & 1024 & 1024 & 2048 \\
        Batch Size & 16 & 16 & 16 \\
        \bottomrule
        \toprule
        \rowcolor{gray!25}
        \textbf{Training Setting} & \textbf{Qwen1.5-MoE-A2.7B} & \textbf{DeepSeekV2-Lite} & \textbf{Qwen3-30B-A3B} \\
        \midrule
        Learning Rate & $5\times 10^{-5}$ & $5\times 10^{-5}$ & $5\times 10^{-5}$ \\
        Learning Rate Schedule & Linear & Linear & Linear \\
        Load Balancing Loss Coefficient ($\alpha$) & $1\times 10^{-3}$ & $1\times 10^{-3}$ & $1\times 10^{-3}$ \\
        Per Device Batch Size & 32 & 32 & 32 \\
        Number of GPUs & 32 & 32 & 64 \\
        Max Token Length & 4096 & 4096 & 4096 \\
        Warm up ratio & 0.03 & 0.03 & 0.03 \\
        Number of Epochs & 2 & 2 & 2 \\

        \bottomrule
    \end{tabular}
    }
    \label{tab:moe_model_and_infer_settings}
\end{table}

We do experiments on three representative MoE models: Qwen1.5‑MoE‑A2.7B \citep{bai2023qwen}, DeepSeekV2‑Lite \citep{liu2024deepseek}, and Qwen3‑30B‑A3B \citep{yang2025qwen3}.
\begin{itemize}
    \item \textbf{Qwen1.5-MoE-A2.7B:} Each token activates $4$ shared experts and $4$ routed experts (out of $60$) in each layer.
    \item \textbf{DeepSeekV2-Lite:} Each token activates $2$ shared experts and $6$ routed experts (out of $64$) in each layer.
    \item \textbf{Qwen3-30B-A3B:} Each token activates $8$ routed experts (out of $128$) in each layer.
\end{itemize}

More details can be found in Table~\ref{tab:moe_model_and_infer_settings}.

\subsubsection{Hyper-Parameters}

Tables~\ref{tab:moe_model_and_infer_settings} summarize the main configurations for all MoE models studied in this work. 
All trainings and evaluations are performed on NVIDIA H20 GPUs

\subsubsection{Benchmarks}

\begin{table}[!h]
    \centering
    \caption{Overview of OpenCompass tasks used for evaluation.}
    \small
    \begin{tabular}{p{2.3cm} p{2.5cm} p{7.5cm}}
        \toprule
        \textbf{Task} & \textbf{Domain/Format} & \textbf{Description / Example} \\
        \midrule
        \textbf{Math} \citep{hendrycksmath2021} & Reasoning / Open-Ended & A dataset of high school-level mathematical problems requiring step-by-step solutions. \newline \\
        & & \textit{Example:} A positive multiple of 45 less than 1000 is randomly selected. What is the probability that it is a two-digit integer? Express your answer as a common fraction.\\
        \midrule
        \textbf{GSM8K} \citep{cobbe2021gsm8k} & Reasoning / Open-Ended & Grade school math word problems with a focus on multi-step reasoning. \newline \\
        
        & & \textit{Example:} Shiloh is 44 years old today.  In 7 years, he will be three times as old as his nephew.  How old is his nephew today?\\
        \midrule
        \textbf{HumanEval} \citep{chen2021evaluating} & Reasoning / Open-Ended & Python programming problems requiring function implementation based on a natural language description.  \newline \\
        & & \textit{Example:} Write a function that returns the sum of two numbers.\\
        \midrule
    
        \textbf{MMLU} \citep{li2023cmmlu} & Knowledge / Multiple-Choice & A massive multitask test consisting of multiple-choice questions from various branches of knowledge. \newline \\
        & & \textit{Example:} Who set the world record for the mile race in 1886? A. R Bannister, B. S Coe, C. J DiMaggio, D. WG George \\
        \midrule
        
        \textbf{CEVAL} \citep{huang2023ceval} & Knowledge / Multiple-Choice & A comprehensive Chinese evaluation suite for foundation models.\newline \\
        & & \textit{Example:} \begin{CJK}{UTF8}{gbsn}下列各项中，应征收资源税的是\_\_\_\_\_。	A. 人造石油 B. 某商贸企业零售的煤炭 C. 开采铁矿石同时开采的锰矿 D. 某联合企业进口的石油  \end{CJK} \\
        \midrule

        \textbf{CMMLU} \citep{li2023cmmlu} & Knowledge / Multiple-Choice & A comprehensive Chinese multi-subject exam benchmark with 57 subjects.\newline \\
        & & \textit{Example:}  
        \begin{CJK}{UTF8}{gbsn}关系数据库中数据的逻辑结构是（A）树结构（B）维度表（C）层次结构（D）形状结构
        \end{CJK}  
        \\
        \midrule
        
        \textbf{BoolQ} \citep{clark2019boolq} & CommonSense / Multiple-Choice (Yes/No) & Reading comprehension questions with yes/no answers based on a passage.\newline \\
        & & \textit{Example:} Property tax -- Property tax or `house tax' is a local tax ... Is house tax and property tax are same?\\
        \midrule
        \textbf{CommonSenseQA} \citep{talmor-etal-2019-commonsenseqa} & CommonSense / Multiple-Choice & A new multiple-choice question answering dataset that requires different types of commonsense knowledge to predict the correct answers . \newline\\
        & & \textit{Example:} Sammy wanted to go to where the people were. Where might he go? A. race track, B. populated areas, C. the desert, D. apartment, E. roadblock."
        \\
        \bottomrule
    \end{tabular}
    \label{tab:opencompass_tasks}
\end{table}

For accuracy comparison, we select a diverse set of tasks from the OpenCompass~\citep{2023opencompass} benchmark, covering multiple domains: \texttt{\textbf{Reasoning}} (MATH~\citep{hendrycksmath2021}, GSM8K~\citep{cobbe2021gsm8k}, and Human Eval~\citep{chen2021evaluating}); \texttt{\textbf{Knowledge}} (MMLU~\citep{hendryckstest2021}, CEVAL~\citep{huang2023ceval}, and CMMLU~\citep{li2023cmmlu}); and \texttt{\textbf{CommonSense}} (CommonsenseQA~\citep{talmor-etal-2019-commonsenseqa}, BoolQ~\citep{clark2019boolq} ). Details and examples of these tasks are provided in Table~\ref{tab:opencompass_tasks}.

For acceleration comparison, we use vLLM~\citep{kwon2023efficient} as the inference framework. Each model is deployed on a single GPU, and we record the \textit{Time per Output Token} (TPOT, in ms) across different \textit{Queries per Second}~(QPS), the \textit{Time To First Token}~(TTFT, in ms) in 32 QPS (high-computing scenarios), and the offline \textit{Throughput}~(samples/s). The input and output sequence lengths are fixed at $128$ and $32$ tokens, respectively, and each test processes a total of 5,000 samples.

\rebuttal{\subsection{Training Dynamics}}
\label{sec:training_dynamics}

\begin{figure}[!h]
    \centering
    \includegraphics[width=\linewidth]{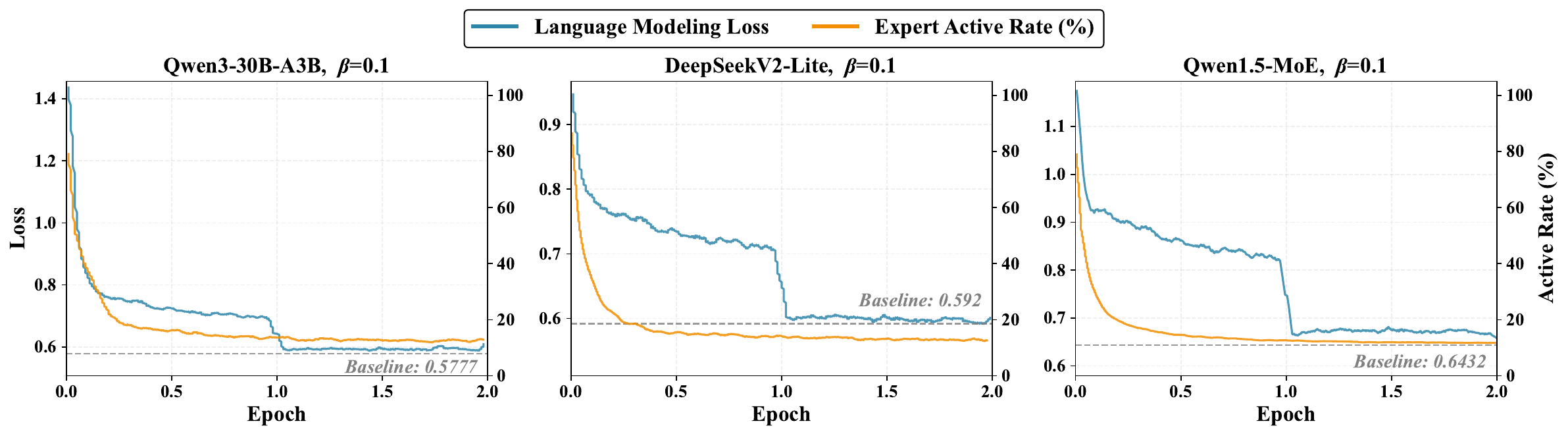}
    \caption{\rebuttal{Training curves of BEAM ($\beta=0.1$) on three MoEs. Blue: language modeling loss (left axis). Orange: expert active rate (right axis). Gray dashed line: SFT baseline loss without BEAM.}}
    \label{fig:training_curves}
\end{figure}

\rebuttal{Figure~\ref{fig:training_curves} shows the language modeling loss and expert active rate during BEAM training ($\beta=0.1$) on all three models. The gray dashed line indicates the converged loss of standard SFT without BEAM.
Two observations emerge.
First, BEAM's language modeling loss converges to a level comparable to the SFT baseline across all models, confirming that the mask router and sparsity regularization do not compromise model capacity.
Second, expert sparsification concentrates in the first $\sim$0.5 epoch, where the active rate drops sharply from near 100\% to a stable plateau. The remaining training focuses on optimizing the language modeling objective under the learned sparsity pattern.}

\subsection{More Baseline Comparison}
\label{sec:more_baseline}

\rebuttal{We additionally compare BEAM with DynMoE~\citep{guo2024dynamic}, which replaces hard Top-K routing with sigmoid-gated expert selection based on token-expert affinity. Table~\ref{tab:dynmoe_comparison} summarizes the results on all three evaluated models.}

\begin{table}[!htbp]
    \centering
    \caption{\rebuttal{Performance comparisons of DynMoE and BEAM relative to the original models.}}
    \label{tab:dynmoe_comparison}
    \resizebox{\linewidth}{!}{
    \begin{tabular}{l|c|ccc|ccc|cc|c}
    \toprule
    \rowcolor{gray!25}
     &
    & \multicolumn{3}{c|}{\textbf{Reasoning}}
    & \multicolumn{3}{c|}{\textbf{Knowledge}}
    & \multicolumn{2}{c|}{\textbf{CommonSense}}
    &   \\

    \noalign{\vspace{-0.85mm}}
    \cmidrule(lr){3-5}  \cmidrule(lr){6-8} \cmidrule(lr){9-10} \noalign{\vspace{-0.85mm}}
    \rowcolor{gray!25}
    \multirow{-2}{*}{\textbf{Methods \textbackslash Tasks}}  &
    \multirow{-2}{*}{\makecell[c]{\textbf{Avg. K} } }&
    \textbf{MATH} & \textbf{GSM8K} & \textbf{H\_Eval} & \textbf{MMLU} & \textbf{CEVAL} & \textbf{CMMLU} & \textbf{BoolQ} & \textbf{CSQA} & \multirow{-2}{*}{\makecell[c]{\textbf{Avg.} \\ (Acc. $\uparrow$)  } }  \\
    \midrule
    \rowcolor{blue!5}
    \multicolumn{11}{c}{\textit{Qwen1.5-MoE-A2.7B}} \\
    Qwen1.5-MoE$_{~\text{K=4}}$ & 4.00 & 23.04 & 57.47 & 50.61 & 59.28 & 74.15 & 75.18 & 72.63 & 81.33 & 61.71 \\
    \rowcolor{gray!10}
    DynMoE & 30.06 & 10.72 & 45.03 & 31.10 & 42.94 & 37.60 & 39.35 & 60.89 & 61.51 & 41.14 \\
     \textbf{BEAM}$_{~(\beta=0.01)}$ & 1.56 & 24.78 & 55.50 & 53.05 & 58.75 & 70.47 & 69.18 & 78.32 & 80.84 & 61.36 \\
    \midrule
    \rowcolor{blue!5}
    \multicolumn{11}{c}{\textit{Qwen3-30B-A3B}} \\
    Qwen3-30B-A3B$_{~\text{K=8}}$ & 8.00 & 58.28 & 88.02 & 82.93 & 81.80 & 83.56 & 83.69 & 86.76 & 86.24 & 81.41 \\
    \rowcolor{gray!10}
    DynMoE & 61.66 & 19.46 & 66.49 & 39.02 & 40.64 & 36.42 & 36.51 & 58.50 & 51.11 & 43.52 \\
    \textbf{BEAM}$_{~(\beta=0.01)}$ & 4.23 & 55.16 & 85.52 & 81.71 & 80.09 & 81.46 & 81.53 & 88.07 & 86.40 & 79.99 \\
    \midrule
    \rowcolor{blue!5}
    \multicolumn{11}{c}{\textit{DeepSeekV2-Lite}} \\
    DeepSeekV2-Lite$_{~\text{K=6}}$ & 6.00 & 20.02 & 62.70 & 43.90 & 55.04 & 55.26 & 60.93 & 75.20 & 68.14 & 55.15 \\
    \rowcolor{gray!10}
    DynMoE & 30.50 & 0.04 & 1.36 & 0.00 & 5.11 & 6.63 & 0.41 & 8.81 & 6.39 & 3.59 \\
    \textbf{BEAM}$_{~(\beta=0.01)}$ & 2.61 & 20.36 & 60.27 & 46.95 & 56.65 & 54.12 & 60.42 & 76.57 & 65.11 & 55.06 \\
    \bottomrule
    \end{tabular}
    }
\end{table}

\rebuttal{
DynMoE exhibits substantial instability in the post-training setting, as its learned gating tends to over-activate experts far beyond the original Top-K budget. Specifically, the average number of activated experts rises to 61.66 on Qwen3-30B-A3B (vs. Top-K = 8), 30.06 on Qwen1.5-MoE-A2.7B (vs. Top-K = 4), and 30.50 on DeepSeekV2-Lite (vs. Top-K = 6). This instability also leads to severe accuracy degradation, most notably on DeepSeekV2-Lite, where performance collapses from 55.15 to 3.59 average accuracy.
This is likely because DynMoE completely replaces the original router architecture, making it ill-suited for post-training scenarios where preserving the pretrained routing structure is critical. In contrast, BEAM achieves substantially higher sparsity while retaining over 98\% of the original model's performance across all three models.}

\rebuttal{\subsection{Layer-wise Masking Rank Analysis}}
\label{sec:layerwise_rank}

\begin{figure}[!h]
    \centering
    \includegraphics[width=0.8\linewidth]{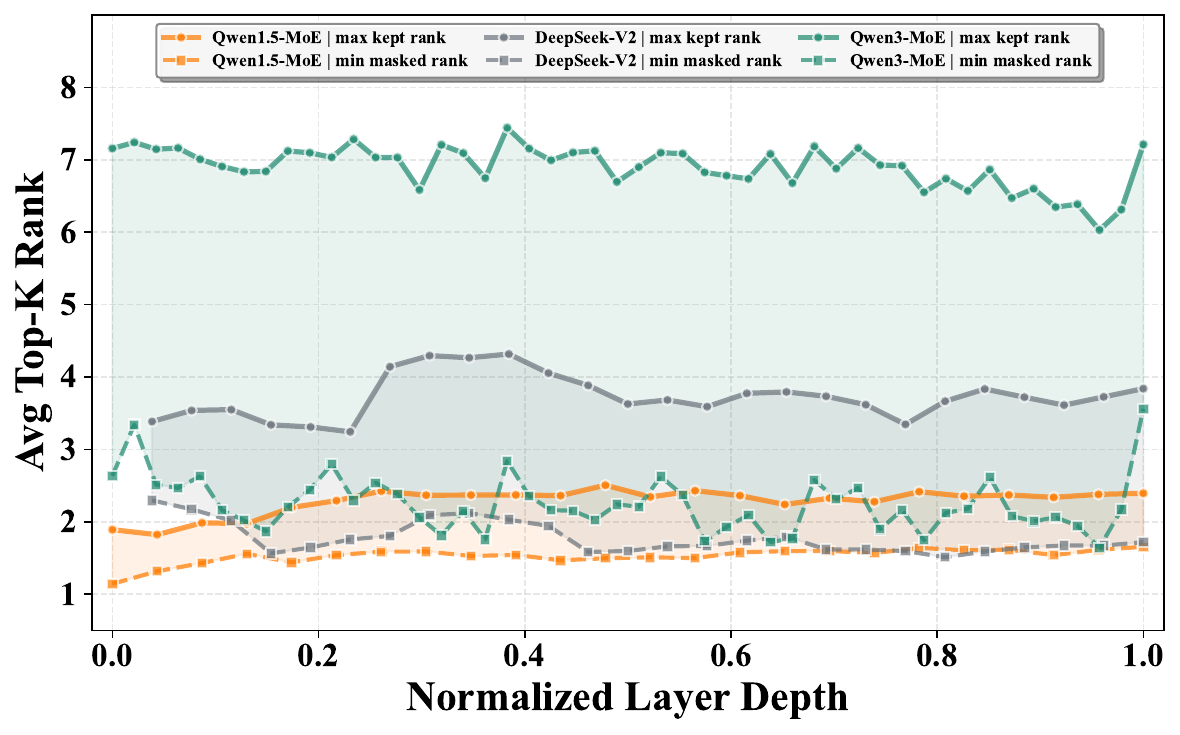}
    \caption{\rebuttal{Layer-wise masking rank analysis across three MoE models. The shaded region between the min masked rank and max kept rank indicates the overlap zone where BEAM's masking decisions are token-dependent.}}
    \label{fig:layerwise_rank}
\end{figure}

\rebuttal{To further investigate whether BEAM's masking decisions follow routing rank, we record the minimum masked rank and maximum kept rank per layer across all three models (Figure~\ref{fig:layerwise_rank}).
Across all layers and models, the min masked rank stays as low as 1--3, meaning that even highly-ranked experts are frequently pruned when redundant for a given token. Meanwhile, the max kept rank extends to the lower end of the Top-K range, confirming that low-ranked experts can be retained when critical. The wide overlap between masked and kept ranks demonstrates that BEAM's decisions are driven by token-expert relevance rather than routing position.}

\subsection{Expert Load-Balance Analysis}
\label{sec:expert_load_balance_analysis}

To investigate BEAM's impact on load balancing in MoE models, we visualize the utilization rates of experts for both original model and BEAM-augmented model during inference in our studied MoE models, as shown in Figure~\ref{fig:load_balance}. The results demonstrate that the BEAM method performs uniform masking across experts, maintaining relatively balanced expert loads even on models such as Qwen3-30B-A3B that contains 128 experts. This finding highlights the applicability of the BEAM method to large-scale expert-parallel MoE models.

\begin{figure}
    \centering
    \includegraphics[width=.9\linewidth]{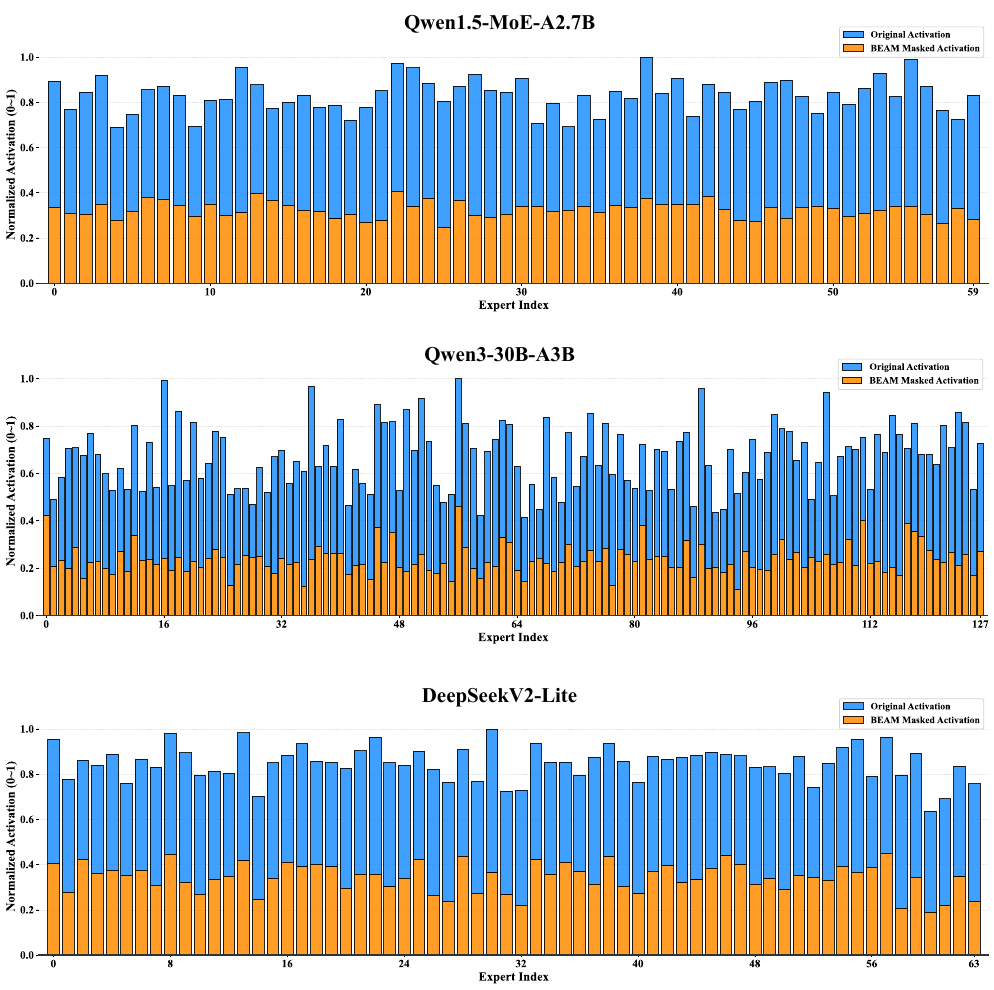}
    \caption{Expert load balance visualization of MoE models before and after BEAM fine-tuning.}
    \label{fig:load_balance}
\end{figure}

\subsection{Task-specific Inference Speed Analysis}
\label{sec:task_specific_analysis}

\begin{table}[!htbp]
\centering
\small
\caption{Task-specific acceleration comparisons on Qwen3-30B-A3B.}
\renewcommand{\arraystretch}{0.9}
\label{tab:speed_test}
\resizebox{\textwidth}{!}{%
\begin{tabular}{lccccccccc}
\toprule
\rowcolor{gray!20}
\textbf{Model} & \textbf{MATH} & \textbf{GSM8K} & \textbf{H\_Eval} & \textbf{MMLU} & \textbf{CEVAL} & \textbf{CMMLU} & \textbf{BoolQ} & \textbf{CSQA} & \textbf{All} \\
\midrule
Qwen3-30B-A3B & 2667s & 329s & 43s & 980s & 77s & 460s & 89s & 46s & 4691s \\
BEAM & 2058s & 247s & 28s & 709s & 54s & 330s & 81s & 36s & 3543s \\
\midrule
\textbf{Speedup} & \textbf{1.30x} & \textbf{1.33x} & \textbf{1.53x} & \textbf{1.38x} & \textbf{1.42x} & \textbf{1.39x} & \textbf{1.10x} & \textbf{1.27x} & \textbf{1.32x} \\
\bottomrule
\end{tabular}%
}
\end{table}

To further evaluate BEAM's acceleration across tasks, we measure the inference speed of the BEAM-augmented Qwen3-MoE model ($\beta=0.1$) and its baseline on several evaluation benchmarks using vLLM on a single NVIDIA H20 GPU. The results are summarized in Table~\ref{tab:speed_test}. BEAM achieves consistent speedups across all tasks, indicating efficiency improvements in both prefill and decoding. 

\subsection{Token-Layer Sparsity Visualization}
\label{sec:token_layer_sparsity_visualization}

We visualize the per-token and per-layer expert activation patterns of Qwen1.5-MoE-A2.7B, DeepSeekV2-Lite, and Qwen3-30B-A3B on the same input prompt: ``In only one sentence, what do you think of the future of AI?''.
The results are shown in Figures~\ref{fig:ds_token_layer_expert_heatmap}, \ref{fig:qwen2_token_layer_expert_heatmap}, and \ref{fig:qwen3_token_layer_expert_heatmap}, which show significant differences in activation patterns across models and layers. Chat template tokens such as ''You are a helpful assistant'' activate almost no experts across all models. Notably, the degree of prefill-decode divergence varies by model: Qwen3 exhibits substantially more expert activation during decoding than prefill, Qwen1.5-MoE shows a moderate increase, while DeepSeekV2-Lite maintains largely consistent activation across both phases. For Qwen models, we also observe that prefill tokens mainly activate experts in shallow layers whereas decode tokens make stronger use of deeper layers, developing an encoder-decoder-like pattern where shallow layers handle knowledge encoding and deeper layers focus on reasoning.

\begin{figure}[!h]
    \centering
    \includegraphics[width=0.7\linewidth]{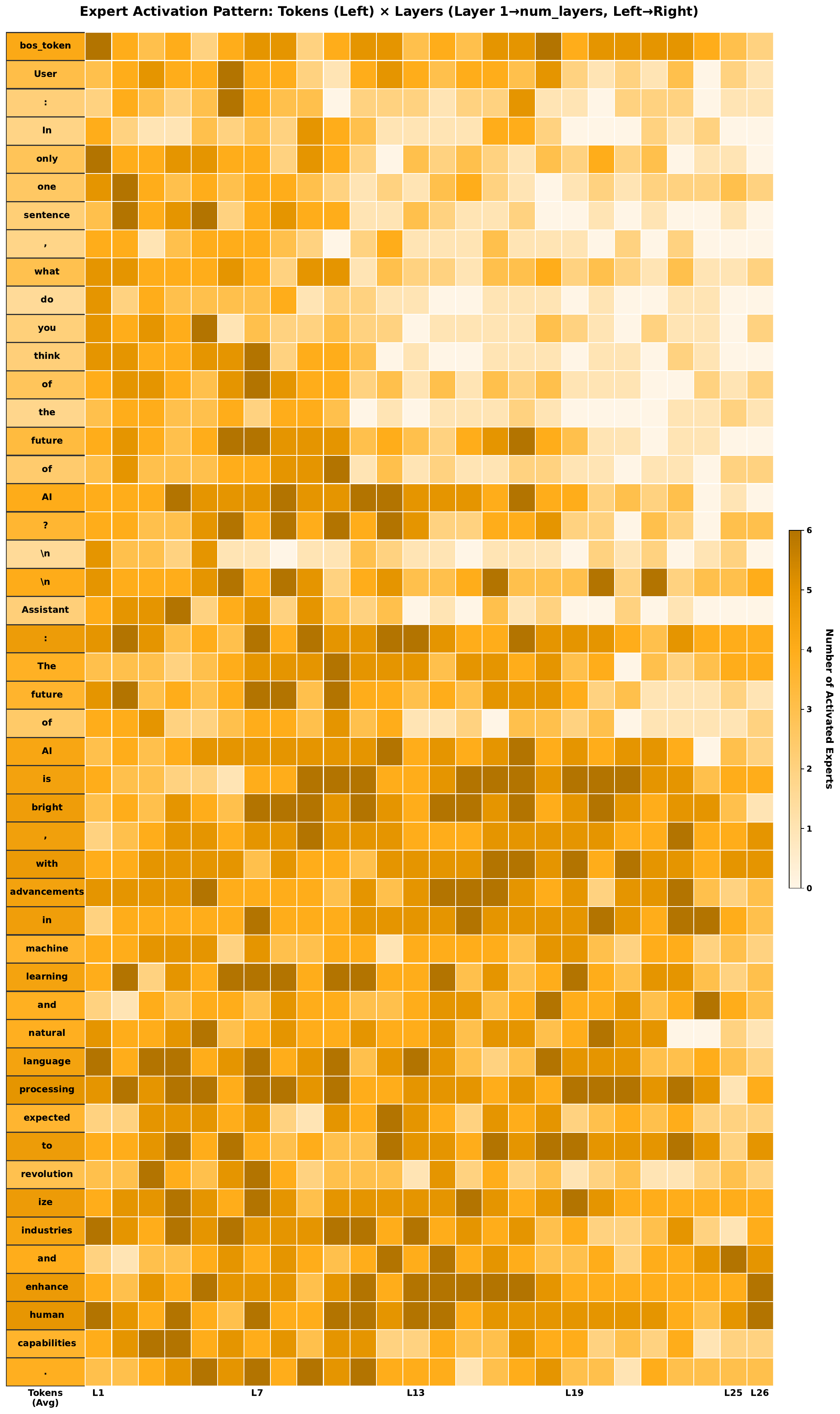}
    \caption{Per-token and per-layer expert activation heatmap for DeepSeekV2-Lite. Each cell indicates the number of activated experts for a token (vertical axis) at a given layer (horizontal axis).}
    \label{fig:ds_token_layer_expert_heatmap}
\end{figure}

\begin{figure}[!h]
    \centering
    \includegraphics[width=0.65\linewidth]{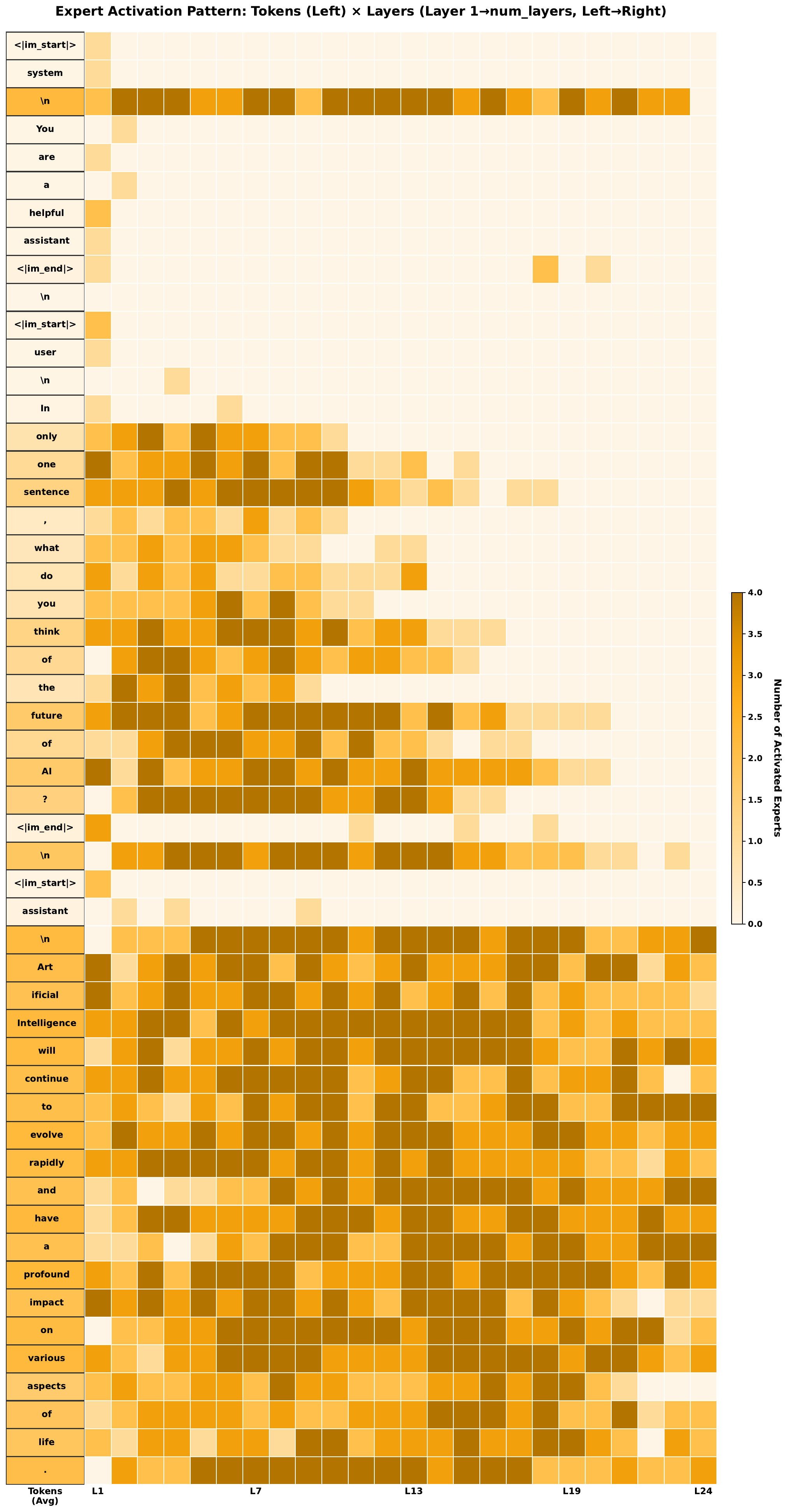}
    \caption{Per-token and per-layer expert activation heatmap for Qwen1.5-MoE-A2.7B. Each cell indicates the number of activated experts for a token (vertical axis) at a given layer (horizontal axis).}
    \label{fig:qwen2_token_layer_expert_heatmap}
\end{figure}

\begin{figure}[!h]
    \centering
    \includegraphics[width=0.93\linewidth]{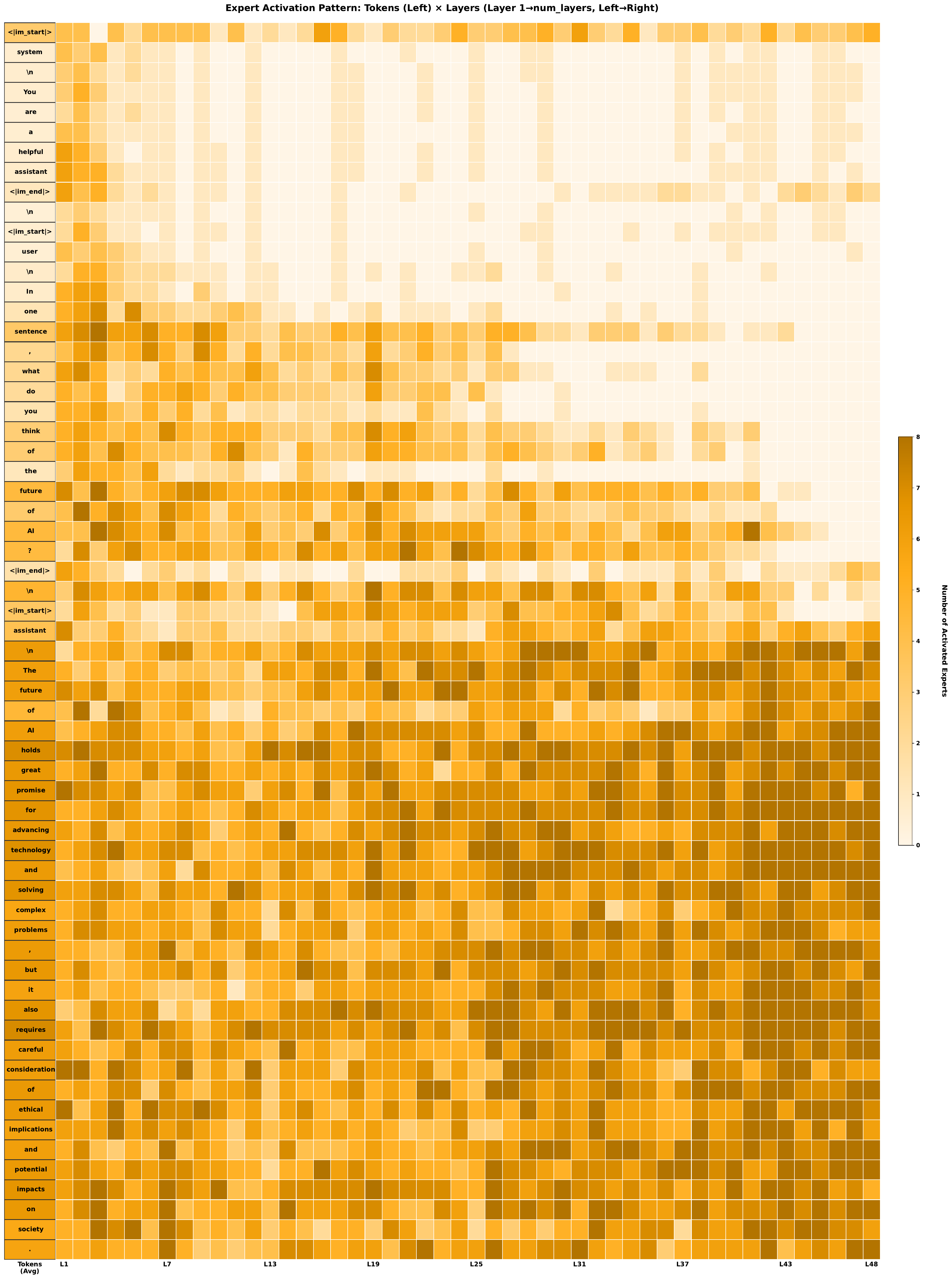}
    \caption{Per-token and per-layer expert activation heatmap for Qwen3-30B-A3B. Each cell indicates the number of activated experts for a token (vertical axis) at a given layer (horizontal axis).}
    \label{fig:qwen3_token_layer_expert_heatmap}
\end{figure}




\end{document}